\documentclass[twoside]{article}
\usepackage[accepted]{aistats2016}
\newif\iffullver
\fullvertrue

\newif\ifsubmit
\submitfalse

\usepackage{natbib}
\bibpunct[; ]{(}{)}{;}{a}{,}{,}
\usepackage{hyperref}
\usepackage{url}
\usepackage{amsmath}
\usepackage{amsthm}
\usepackage{amssymb}
\usepackage[ruled]{algorithm2e}
\usepackage{easybmat}
\usepackage{multirow}
\usepackage{graphicx}
\usepackage{color}
\usepackage{enumitem}

\newtheorem{theorem}{Theorem}
\newtheorem{lemma}{Lemma}
\newtheorem{proposition}{Proposition}

\newtheorem{corollary}{Corollary}

\newcommand{\mathds}[1]{\mathbb{#1}}
\newcommand{\beq}{\begin{equation}}
\newcommand{\eeq}{\end{equation}}
\newcommand{\R}{\mathds{R}}

\newcommand{\cE}{{\mathcal E}}
\newcommand{\cF}{{\mathcal F}}
\newcommand{\cG}{{\mathcal G}}
\newcommand{\cH}{{\mathcal H}}

\newcommand{\cN}{{\mathcal N}}

\newcommand{\cR}{{\mathcal R}}

\newcommand{\cT}{{\mathcal T}}

\newcommand{\cV}{{\mathcal V}}
\newcommand{\cW}{{\mathcal W}}
\newcommand{\cX}{{\mathcal X}}
\newcommand{\cY}{{\mathcal Y}}

\newcommand{\va}{{\mathbf a}}
\newcommand{\vb}{{\mathbf b}}

\newcommand{\vg}{{\mathbf g}}
\newcommand{\vr}{{\mathbf r}}
\newcommand{\vs}{{\mathbf s}}
\newcommand{\vx}{{\mathbf x}}
\newcommand{\vy}{{\mathbf y}}
\newcommand{\vv}{{\mathbf v}}
\newcommand{\vw}{{\mathbf w}}
\newcommand{\vz}{{\mathbf z}}
\newcommand{\vA}{{\mathbf A}}
\newcommand{\vD}{{\mathbf D}}
\newcommand{\vI}{{\mathbf I}}

\newcommand{\vX}{{\mathbf X}}
\newcommand{\vP}{{\mathbf P}}
\newcommand{\vS}{{\mathbf S}}
\newcommand{\vV}{{\mathbf V}}

\newcommand{\vxi}{{\boldsymbol \xi}}
\newcommand{\vlam}{{\boldsymbol \lambda}}
\newcommand{\vtau}{{\boldsymbol \tau}}
\newcommand{\vnu}{{\boldsymbol \nu}}
\newcommand{\Glam}{{\cG}}
\newcommand{\GlamDual}{{\cG^D}}

\newcommand{\dom}{\textnormal{dom}\,}

\newcommand{\wi}[1]{|w|_{(#1)}}

\DeclareMathOperator*{\argmin}{arg\,min}
\DeclareMathOperator*{\argmax}{arg\,max}
\newcommand{\prox}{\textnormal{prox}}

\definecolor{ejc}{RGB}{255,0,0}

\begin{document}

%
\runningtitle{Fast Saddle-Point Algorithm for Generalized Dantzig
  Selector and FDR Control with Ordered l1-Norm}

%

\twocolumn[

\aistatstitle{Fast Saddle-Point Algorithm for Generalized Dantzig
  Selector and FDR Control with the Ordered l1-Norm}

\ifsubmit
\aistatsauthor{ Anonymous Author 1 \And Anonymous Author 2 \And Anonymous Author 3 }
\aistatsaddress{}
\else
\aistatsauthor{ Sangkyun Lee \And Damian Brzyski \And Malgorzata Bogdan }
\aistatsaddress{ Technische Universit\"at Dortmund \And Jagiellonian University \And University of Wroclaw}
\fi

]

\begin{abstract} 
In this paper we propose a primal-dual proximal
extragradient algorithm to solve the generalized Dantzig selector
(GDS) estimation problem, based on a new convex-concave saddle-point
(SP) reformulation. Our new formulation makes it possible to adopt
recent developments in saddle-point optimization, to achieve the
optimal $O(1/k)$ rate of convergence. Compared to the optimal non-SP
algorithms, ours do not require specification of sensitive parameters
that affect algorithm performance or solution quality. We also provide
a new analysis showing a possibility of local acceleration to achieve
the rate of $O(1/k^2)$ in special cases even without strong convexity
or strong smoothness. As an application, we propose a GDS equipped
with the ordered $\ell_1$-norm, showing its false discovery rate
control properties in variable selection.  Algorithm performance is
compared between ours and other alternatives, including the linearized
ADMM, Nesterov's smoothing, Nemirovski's mirror-prox, and the
accelerated hybrid proximal extragradient techniques.
\end{abstract}

\section{INTRODUCTION}

The Dantzig selector~\citep{CanT07} has been proposed as an
alternative approach for penalized regression, mainly in the context
of sparse or group sparse regression in high dimensions. A generalized
Dantzig selector (GDS)~\citep{ChaC14} has been recently proposed
extending the original Dantzig selector, to use any norm $\cR(\cdot)$
for regularization and its dual norm $\cR^D(\cdot)$ for measuring
estimation error. For linear models of the form $\vy = \vX\vw^* +
\vxi$, where $\vy \in \R^{n}$ contains observations, $\vX \in
\R^{n\times p}$ is a design matrix, and $\vxi$ is an i.i.d. standard
Gaussian noise vector, the GDS searches for the best parameter solving
the following problem with a constant $c > 0$:
\begin{equation}\label{eq:gds}
 \min_{\vw \in \R^p} \;\; \cR(\vw) \;\; \text{s.t.} \;\; \cR^D(\vX^T(\vy - \vX\vw)) \le c .
\end{equation}
The original Dantzig selector is attained when
$\cR(\cdot)=\|\cdot\|_1$ and $\cR^D(\cdot) = \|\cdot\|_\infty$.  The
GDS requires to solve a non-separable and non-smooth convex
optimization problem, which does not contain any strongly smooth part
(with Lipschitz continuous gradients) required to apply (accelerated)
proximal gradient methods~\citep{Nes83,BecT09}. Subgradient
methods~\citep{Shor85} can be applied, but their very slow
$O(1/\sqrt{k})$ convergence rate (for an iteration counter $k$) is not
desirable for practical use.

\citet{ChaC14} proposed an algorithm to solve \eqref{eq:gds} based on
a linearized version of alternating direction method of multipliers
(L-ADMM)~\citep{WanB14,WanY12}, of which two subproblems are
simplified to two proximal operations thanks to linearization and fast
projection: regarding the latter, projection was onto the dual ball
defined with $\cR^D(\cdot)$ and therefore can be easily computed via
the proximal operator of $\cR(\cdot)$ and Moreau's
identity~\citep{Roc97}. The algorithm exhibits $O(1/k)$ convergence
rate when its penalty parameter is set to a value at least
$\|\vX\|_2^4$~\citep{ChaC14,WanB14}.  However, its practical
performance tends to be quite sensitive to the parameter, whose best
value is not easy to determine a priori running the algorithm.

Recently, there have been attractive improvements in ADMM, although
they are not applicable to our problem due to their extra
requirements. Local linear convergence has been shown for
ADMM, but for the limited cases of minimizing a quadratic objective
under linear constraints~\citep{Bol13}, or minimizing a sum of
strongly convex smooth functions~\citep{ShiL14}. Accelerated versions
of ADMM recently appeared achieving a better $O(1/k^2)$ rate, however,
with an assumption that the objective is strongly convex in case of
ADMM~\citep{GolO14,KadC15}, or with a smoothness assumption of the
part to be linearized in case of L-ADMM~\citep{OuyC15}.

The GDS problem~\eqref{eq:gds} can also be solved using the smoothing
technique due to \citet{Nes05a}. It is based on creating a smooth
approximation of a non-smooth function by adding a strongly convex
regularizer to the conjugate of the non-smooth function, where the
strong convexity is modulated by a parameter $\mu>0$. It is shown that
the smooth approximation has Lipschitz continuous gradients and
therefore can be optimized via accelerated gradient
methods~\citep{Nes83}.  The smoothing technique achieves $O(1/k)$ rate
of convergence when $\mu = O(\epsilon)$~\citep[Theorem 3]{Nes05a} for
an optimality gap $\epsilon$. However, using small values of $\mu$ to
achieve a near-optimal solution tends to slow down the algorithm quite
significantly in practice. Implementations of Nesterov's smoothing
such as TFOCS~\citep{BecC11} require users to specify this parameter
with only little guidance.

In this paper, we propose a new convex-concave saddle-point (CCSP)
formulation of the GDS, in fact a slightly more generalized version of
it to allow for using any convex function for regularization. Our
reformulation allows us to provide a fast and simple algorithm to find
solutions of GDS instances, achieving the optimal $O(1/k)$ convergence
rate without relying on sensitive parameters affecting convergence or
solution quality. Our algorithm is applied to a new kind of GDS
defined with the ordered $\ell_1$-norm: we prove its false discovery
rate control properties in variable selection, where the norm itself
has been recently studied in other contexts~\citep{BogB13,BogB15,FigN14}.

We show that our proposed algorithm suits better than existing solvers
when high-precision solutions are desired for accurate variable
selection, for example in statistical simulation studies. We denote
the Euclidean norm by $\|\cdot\|$ and inner products by $\langle
\cdot, \cdot \rangle$.

\section{CONVEX-CONCAVE SADDLE-POINT FORMULATION}

\subsection{(More) Generalized Dantzig Selector}

In this paper we consider a slightly more general form of the GDS
problem~\eqref{eq:gds},
\begin{equation}\label{eq:gds.2}
\text{(GDS)} \quad  \min_{\vw \in \R^p} \;\; \cF(\vw) \;\; \text{s.t.} \;\; \GlamDual(\vX^T (\vy - \vX\vw) ) \le 1.
\end{equation}
where $\cF: \R^p \to (-\infty, +\infty]$ is a proper, convex, and
lower-semicontinuous (l.s.c.) function, and $\GlamDual(\cdot)$ is the
dual norm of a norm $\Glam(\cdot)$, possibly parametrized by a vector
$\vlam$. Unlike \eqref{eq:gds}, $\Glam$ is not necessarily the same as
$\cF$, and also $\cF$ does not have to be a norm. Neither $\cF$ nor
$\Glam$ is assumed to be differentiable.

\subsection{Reformulation}

Denoting by $C_{\GlamDual}$ the constraint set of residuals in
\eqref{eq:gds.2},
$$
C_{\GlamDual} := \{ \vr \in \R^p \;:\; \GlamDual(\vr) \le 1 \},
$$
and using an indicator function $\vartheta_{C_{\GlamDual}}(\vr)$, which
returns $0$ if $\vr \in C_{\GlamDual}$ or $+\infty$ otherwise, it is
trivial to see the GDS problem~\eqref{eq:gds.2} can be restated as,
\begin{equation}\label{eq:gds.reg}
 \min_{\vw \in \R^p} \;\; \cF(\vw) + \vartheta_{C_{\GlamDual}} (\vX^T (\vy - \vX\vw)) .
\end{equation}
Now, we invoke a simple lemma to replace the indicator function with
its adjoint form.

\begin{lemma}\label{lem:conj.ind}
For any $\vw \in\R^p$, we have
$$
 \vartheta_{C_{\GlamDual}} (\vX^T (\vy - \vX\vw)) 
= \max_{\vv \in \R^p} \, \left\langle \vA  \begin{bmatrix} \vy \\ \vw \end{bmatrix}, \vv \right\rangle - \Glam(\vv) ,
$$
where $\vA:= \vX^T\begin{bmatrix} \vI_{n} & -\vX\end{bmatrix} \in
\R^{p\times (n+p)}$ and $\vI_{n}$ is the $n\times n$ identity matrix.
\end{lemma}
\begin{proof}
  Since $\vartheta_{C_{\GlamDual}}$ is an indicator function on a
  closed set, we have $\vartheta_{C_{\GlamDual}} (\cdot) =
  \vartheta^{\star\star}_{C_{\GlamDual}} (\cdot)$ with the biconjugation
$$
  \vartheta^{\star\star}_{C_{\GlamDual}} (\vr)  = \sup_{\vv \in \R^p} \{ \langle \vr, \vv \rangle - \vartheta^\star_{C_\GlamDual}(\vv) \} .
$$
Also, from conjugacy, $\vartheta^\star_{C_\GlamDual}(\cdot) = \sup_{\vw'
  \in \R^p} \; \langle \vw', \cdot \rangle - \vartheta_{C_\GlamDual} (\vw') =
\max_{\vw': \GlamDual(\vw') \le 1} \langle \vw', \cdot \rangle $, 
which is by definition the dual norm of $\GlamDual(\cdot)$,
i.e., $\Glam(\cdot)$. The result follows when we set $\vr = \vX^T (\vy
- \vX\vw)$.
\end{proof}

The following convex-concave saddle-point reformulation of the
GDS~\eqref{eq:gds.2} follows when we apply the above lemma to
\eqref{eq:gds.reg},
\begin{equation}\label{eq:gds.sp}
\text{(GDS-SP)} \;\; \min_{\vw\in\R^p}\max_{\vv \in \R^p} \,
\left\langle \vA \begin{bmatrix} \vy \\ \vw \end{bmatrix}, \vv \right\rangle 
+ \cF(\vw) - \Glam(\vv) .
\end{equation}



This reformulation allows us to benefit from recent developments in
saddle-point optimization, including our algorithm discussed
later. Hereafter, we assume that both $\cF$ and $\Glam$ are simple, so
that their {\em proximal operator}, defined below for $\cF$, can be
computed efficiently:
\begin{align*}
      \prox_{\cF} (\vz) 
     := \argmin_{\vw'} \; \left\{ \frac{1}{2}\|\vw' - \vz\|^2 + \cF(\vw') \right\} .
\end{align*}
Note that it suffices to meet this requirement for either $\cF$ or its
conjugate $\cF^\star$ (similarly for $\Glam$ or $\Glam^\star$), since
the prox operation for one can be computed by that of the other by
Moreau's identity~\citep{Roc97}, i.e., $ \vz = \prox_{\cF} (\vz) +
\prox_{\cF^\star} (\vz) $.

\subsection{Related Works}

It is worthwhile to note that the Tikhonov-type formulation of the
GDS~\eqref{eq:gds.reg} is closely related to the popular regularized
estimation problems in machine learning and statistics,
$$
 \min_{\vw \in \cW} \;\; \cF(\vw) + \cE(\vD\vw) ,
$$
where $\vD$ is a data matrix and $\cE$ is a proper convex
l.s.c. loss function. Using biconjugation of $\cE$ similarly to the proof
of Lemma~\ref{lem:conj.ind}, this can be reformulated as the following
convex-concave saddle-point problem,
$$
 \min_{\vw \in \cW}\max_{\vv \in \cV} \;\; \phi(\vw,\vv) := \langle \vD\vw, \vv \rangle + \cF(\vw)  - \cE^\star(\vv) ,
$$
given that a maximizer in $\cV$ can be attained (in our case it is
true as $\cV = \{\vw': \GlamDual (\vw') \le 1\}$ is compact).

This type of reformulation has been studied quite recently in machine
learning to design new algorithms. For example, \citet{ZhaX15}
proposed a stochastic primal-dual coordinate descent (SPDC) algorithm
based on a saddle-point reformulation for the case when $\cF$ is
strongly convex and $\cE$ is a sum of smooth loss functions with
Lipschitz continuous gradients, in which case the conjugate
$\cE^\star$ becomes strongly convex~\citep[Proposition 12.60]{Roc04}:
both do not hold in case of the GDS. Although SPDC can be extended for
nonsmooth cases by augmenting strongly convex terms, then it shares
similar issues to Nesterov's smoothing that a parameter needs
to be specified depending on an unknown quantity $\|\vw^*\|$ when
$(\vw^*,\vv^*)$ is a saddle point.

Another example is \citet{TasLJ06} who considered a saddle-point
reformulation of max-margin estimation for structured output
prediction and proposed an algorithm more memory efficient than its
quadratic program alternative, based on the dual extragradient
technique of \citet{Nes07d}. The dual extragradient method itself is
closely related to our method, but it additionally requires that both
$\cF$ and $\cE$ are smooth with Lipschitz continuous gradients to
achieve an ergodic $O(1/k)$ convergence rate, or both $\partial \cF$
and $\partial \cE$ are bounded in which case the algorithm exhibits a
slower $O(1/\sqrt{k})$ rate.


Extragradient techniques to handle the CCSP problems are of our
particular interest. The mirror-prox method~\citep{Nem04} has extended
one of the earliest extragradient algorithm of \citet{Kor76},
establishing the $O(1/k)$ ergodic (in terms of averaged iterates) rate
of convergence with two proximal operations per iteration. This method
however requires to choose stepsizes carefully with the knowledge of
$L = \|\vA\|$. \citet{Tse08} suggested a line search procedure to find
better estimates of $L$, which requires to compute two extra proximal
operations per line search step.

The hybrid proximal extragradient (HPE)
algorithm~\citep{SolS99a,SolS99b} belongs to another family of
extragradient methods that can be seen as a generalization of
Korpelevich's method and some extensions~\citep{MonS11}, and can solve
CCSP problems with the same $O(1/k)$ ergodic convergence rate. In each
iteration of the HPE framework, an extragradient is computed by
solving a subproblem with controlled inaccuracy. The subproblem itself
can be solved using an accelerated method similar to Nesterov's
smoothing~\citep{HeM14} using three proximal operations in each inner
iteration. A pitfall however is that the accuracy of solving the
subproblem tends to affect the overall runtime.

Recently, \citet{ChaP11} proposed a simple extragradient technique
with $O(1/k)$ ergodic convergence rate, which is quite different in
its nature to the aforementioned extragradient methods, although it
may look similar to Nesterov's dual extragradient
technique~\citep{Nes07d}. In \citet{ChaP11}, proximal steps are taken
in each of the primal and the dual spaces, then a linear gradient
extrapolation is considered either in the primal or in the dual. We
base our algorithm on this technique, since it has been the fastest
with the smallest variations in runtime to solve the problem of our
interest in its saddle-point reformulation~\eqref{eq:gds.sp}. Both
properties were desired in particular for studying statistical
properties of the GDS based on random simulations.

\section{ALGORITHM}

Solving the GDS-SP problem~\eqref{eq:gds.sp}, we assume that there
exists a saddle point $(\vw^*,\vv^*)$ satisfying the conditions
\begin{equation}\label{eq:sp}
\begin{aligned}
\vA \begin{bmatrix} \vy \\ \vw^* \end{bmatrix} = \vX^T\vy - \vX^T\vX \vw^* &\in \partial \cG(\vv^*), \\
- (\vA_{[\cdot, (n+1):(n+p)]})^T \vv^* = \vX^T\vX \vv^*  &\in \partial \cF(\vw^*) 
\end{aligned}
\end{equation}
where $\partial \cF$ and $\partial \cG$ are the subdifferentials of
$\cF$ and $\cG$, respectively. Denoting the objective by $\phi$, i.e.,
$$
 \phi(\vw,\vv) := \left\langle \vA \begin{bmatrix} \vy \\ \vw \end{bmatrix}, \vv \right\rangle + \cF(\vw) - \cG(\vv),
$$
the above conditions~\eqref{eq:sp} imply that the following saddle-point inequality
holds for any $(\vw,\vv)$,
$$
\phi(\vw^*,\vv) \le \phi(\vw^*,\vv^*) \le \phi(\vw,\vv^*) .
$$


We present our primal-dual saddle-point (PDSP) algorithm in
Algorithm~\ref{alg:1}, which solves the CCSP formulation of the GDS
problem~\eqref{eq:gds.sp}.

We define the primal-dual gap, following \citet{ChaP11}, restricted to
the set $\cX\times \cY$,
\begin{align*}
 \cT_{\cX\times \cY}(\vw,\vv) := \max_{\vv' \in \cY}\; \left\{ \langle \vA \begin{bmatrix} \vy \\ \vw \end{bmatrix}, \vv' \rangle + \cF(\vw) - \cG(\vv') \right\} \\
 - \min_{\vw' \in \cX} \; \left\{ \langle \vA \begin{bmatrix} \vy \\ \vw' \end{bmatrix}, \vv \rangle + \cF(\vw') - \cG(\vv) \right\} .
\end{align*}
When $\cX\times \cY$ contains a saddle-point $(\vw^*,\vv^*)$
satisfying \eqref{eq:sp}, then it is easy to check that
\begin{align*}
 \cT_{\cX\times \cY}(\vw,\vv) \ge \left\{ \langle \vA \begin{bmatrix} \vy \\ \vw \end{bmatrix}, \vv^* \rangle + \cF(\vw) - \cG(\vv^*) \right\} \\
 - \left\{ \langle \vA \begin{bmatrix} \vy \\ \vw^* \end{bmatrix}, \vv \rangle + \cF(\vw^*) - \cG(\vv) \right\} \ge 0.
\end{align*}

\begin{algorithm}[!t]
\SetKwComment{tpa}{(}{)}
\SetKwInOut{Input}{input}
\SetKwInOut{Output}{output}
\SetKwInOut{Data}{Data}
\SetKwInOut{Init}{Initialize}
\SetKwInOut{Params}{Params}
\caption{Primal-Dual Saddle-Point (PDSP)}\label{alg:1}
\Data{$\vX \in \R^{n\times p}$, $\vy \in \R^n$, $L = \|\vX^T\begin{bmatrix} \vI_{n} & -\vX\end{bmatrix}\|$\;}
\Init{$(\vw_0, \vv_0) \in \R^p \times \R^p$, $\vw'_0 = \vw_0$\;}
\Params{$\tau_0>0$, $\sigma_0>0$ satisfying $\tau_0\sigma_0 L^2 \le 1$, $\gamma \ge 0$ : strong convexity modulus of $\cG$\;} 
\smallskip
\For {$k=0,1,2,\dots$} {
$$
\begin{aligned}
 \vv_{k+1} &= \prox_{\sigma_k \Glam} \left(\vv_k + \sigma_k (\vX^T\vy - \vX^T\vX\vw_k') \right), \nonumber\\
 \vv'_{k+1} &= \vv_{k+1} \text{ (or $2\vv_{k+1}$, see Section~\ref{sec:strong.conv})}, \nonumber \\
 \vw_{k+1} &= \prox_{\tau_k \cF} \left(\vw_k + \tau_k \vX^T\vX \vv'_{k+1} \right), \nonumber \\
 \theta_k &= 1/\sqrt{1+2\gamma \tau_k}, \nonumber \\
 \tau_{k+1} &= \theta_k \tau_k, \;\; \sigma_{k+1} = \sigma_k / \theta_k, \nonumber\\
 \vw'_{k+1} &= \vw_{k+1} + \theta_k (\vw_{k+1} - \vw_k) . \nonumber 
\end{aligned}
$$
 Check (both if $\gamma=0$, only pointwise if $\gamma>0$):
\begin{itemize}[itemsep=0px,topsep=0pt]
 \item[$-$] Pointwise convergence of $(\vw_{k+1},\vv_{k+1})$\;
 \item[$-$] Ergodic convergence of $(\overline \vw_{k+1},\overline \vv_{k+1}) = \frac{1}{k+1} \sum_{i=1}^{k+1} (\vw_i, \vv_i)$\;
\end{itemize}
 }
\end{algorithm}

\begin{theorem}\label{thm:1}
  Suppose that $(\vw^*,\vv^*)$ is a saddle-point of the GDS-SP
  problem~\eqref{eq:gds.sp}. Then the iterates $(\vw_k,\vv_k)$
  generated by Algorithm~\ref{alg:1} with $\gamma=0$ and $\theta_k=1$
  for all $k$ (therefore $\tau_k=\tau_0$ and $\sigma_k = \sigma_0$)
  satisfy the following properties:
\begin{enumerate}
 \item[(a)]  $(\vw_k,\vv_k)$ is bounded for any $k$, i.e.,
\begin{align*}
&\frac{\|\vw_k - \vw^*\|^2}{\tau_0} + \frac{\|\vv_k - \vv^*\|^2}{\sigma_0} \\
& \quad \le C \left( \frac{\|\vw_0-\vw^*\|^2}{\tau_0} + \frac{\|\vv_0 - \vv^*\|^2}{\sigma_0} \right) 
\end{align*}
for a constant $C \le 1/(1-\tau_0\sigma_0 L^2)$.

\item[(b)] For averaged iterates $\overline \vw_k =
  \frac{1}{k}\sum_{i=1}^k \vw_i$ and $\overline \vv_k =
  \frac{1}{k}\sum_{i=1}^k \vv_i$, we have
$$
 \cT(\overline \vw_k, \overline \vv_k) \le \frac{1+C}{k}
 \left( \frac{\|\vw^*-\vw_0\|^2}{2\tau_0} + \frac{\|\vv^*-\vv_0\|^2}{2\sigma_0} \right).
$$
Moreover, limit points of $(\overline \vw_k, \overline \vv_k)$ are
saddle-points of \eqref{eq:gds.sp}.

\item[(c)] There exists a saddle-point $(\hat\vw,\hat\vv)$ of
  \eqref{eq:gds.sp} such that $(\vw_k, \vv_k) \to (\hat\vw, \hat \vv)$
  as $k\to \infty$.
\end{enumerate}
\end{theorem}
\begin{proof}
  Define augmentations of $\vw$'s with $\vy$, e.g. $\vz_k :=
  [\vy; \vw_k] \in \R^{n+p}$, and define $\cH(\vz) = \cH(\vy,\vw) :=
  \cF(\vw)$. Using these, the GDS-SP problem~\eqref{eq:gds.sp} can be written
  equivalently as
$$
 \min_{\vz \in \R^{p+n}} \max_{\vv\in\R^p} \;\; \langle \vA \vz, \vv \rangle + \cH(\vz) - \cG(\vv).
$$
Then the result essentially follows from Theorem 1 of
\citet{ChaP11}. For completeness, we provide the full proof in the
supplementary material, part of which will be used to show
Theorem~\ref{thm:2} as well.
\end{proof}

The {\em ergodic convergence} in Theorem~\ref{thm:1} part (b) indicates that
the primal-dual gap converges with $O(1/k)$ rate for {\em averaged
  iterates}, which is known to be the best rate in general
convex-concave saddle-point
solvers~\citep{Nem04,Tse08,SolS99a,SolS99b,HeM14}. 

The part (c) states {\em pointwise} convergence without averaging,
where its rate is unknown: one can conjecture from related
extragradient methods, e.g. \citet[Theorem 3.4]{HeM14}, that the
convergence might be at a slower rate of $O(1/\sqrt{k})$, but it is
only an educated guess since the methods are not exactly the same. In
fact, in our experiments the iterates tend to converge faster than
averaged iterates, which we will discuss further in detail later.

The part (a) of the above theorem is indeed crucial for our discussion
in the sequel. (We note that similar boundedness results are available
for some related methods, e.g. \citet{Nem04,Tse08}, but not for
all). In particular, in many sparse regression scenarios in
high-dimensions, we expect that $\|\vw^*\|$ may not be very large due
to the small support size (the number of nonzero components) of a true
signal.
As our algorithm naturally starts from the zero vector ($\vw_0 =
\mathbf{0}$), it is therefore likely from Theorem~\ref{thm:1}~(a),
with some proper values of $\tau_0$ and $\sigma_0$, that $\|\vw_k -
\vw^*\|$ (or even $\|\vw_k\|$) would be small as well, although we need
more information about $\|\vv_0 - \vv^*\|$ to say it definitely.

\subsection{Local Strong Convexity and Acceleration}\label{sec:strong.conv}

When $\cF$ or $\cG$ is strongly convex, it can be shown that
Algorithm~\ref{alg:1} exhibits a faster $O(1/k^2)$ {\em pointwise}
convergence rate due to \citet{ChaP11}, using the same trick as in the
proof of Theorem~\ref{thm:1}.

Here, we claim that such an acceleration is also possible, at least
locally, without strong convexity.  Let us focus on $\cF$, since our
arguments here can be equally applied for $\cG$. When $\cF$ is
strongly convex, it satisfies
$$
 \cF(\vw') \ge \cF(\vw) + \langle \vg, \vw'-\vw \rangle + \frac{\gamma}{2} \|\vw'-\vw\|^2, \;\; \vg \in \partial \cF(\vw),
$$
for some modulus $\gamma>0$ and for any $\vw', \vw \in \dom \cF$.

Suppose that $\cF$ is {\em not} strongly convex (i.e., $\gamma=0$), as
in the general GDS cases~\eqref{eq:gds.sp}. Also, suppose that $\cF$ is
indeed a norm, so that it satisfies the (reverse) triangle inequality,
$ \cF(\vw^*) - \cF(\vw_k) \le \cF(\vw^*-\vw_k)$, for a solution
$\vw^*$ and an iterate $\vw_k$ of Algorithm~\ref{alg:1}. If
$\cF(\vw^*-\vw_k)$ is bounded so that $\cF(\vw^* - \vw_k) \le c
\|\vw^*-\vw_k\|$ holds for some $c>0$, where the right-hand side is
bounded due to Theorem~\ref{thm:1} (a), then we can find constants
$\bar c, \delta >0$ such that
\begin{equation}\label{eq:rev.tri}
 \cF(\vw^*) - \cF(\vw_k) \ge \bar c \cF(\vw^*-\vw_k) \ge \delta \|\vw^*-\vw_k\|^2,
\end{equation}
for all $k \ge k_0$, with some $k_0 > 0$ (note that $\vw_k \to \vw^*$
due to Theorem~\ref{thm:1} (c)). Together with the inequality from the
convexity of $\cF$, i.e., $\cF(\vw^*) \ge \cF(\vw_k) + \langle \vg,
\vw^* - \vw_k \rangle$ with $\vg \in \partial \cF(\vw_k)$, it follows
that
\begin{equation}\label{eq:strong.conv}
 \cF(\vw^*) \ge \cF(\vw_k) + \frac{1}{2} \langle  \vg, \vw^* - \vw_k \rangle + \frac{\delta}{2} \|\vw^*-\vw_k\|^2 .
\end{equation}
Comparing to the above inequality of strong convexity, this provides
us a weaker notion of strong convexity in the region where
\eqref{eq:rev.tri} holds. We show that this is enough to establish a
local accelerated pointwise convergence rate even in non-strongly
convex cases:

\begin{theorem}\label{thm:2}
  Let the iterates $(\vw_k,\vv_k)$ be generated by
  Algorithm~\ref{alg:1} with the choices of $\tau_0$ and $\sigma_0$
  such that $2\tau_0\sigma_0L^2 = 1$, and $\vv'_{k+1} =
  2\vv_{k+1}$. Suppose that the local strong
  convexity~\eqref{eq:strong.conv} holds for $\cF$ with a constant
  $\delta>0$ about $\vw_k$, $\forall k\ge k_0$ with some $k_0>0$. Then
  for a saddle-point $(\vw^*,\vv^*)$ of the GDS-SP
  problem~\eqref{eq:gds.sp}, there exists $k_1 \ge k_0$ depending on
  $\epsilon\ge 1$ and $\delta\tau_0$ such that for all $k\ge k_1$,
$$
 \| \vw^* -\vw_k \|^2 \le \frac{4\epsilon}{k^2} \left( \frac{\|\vw^*-\vw_0\|^2}{4\delta^2 \tau_0^2} + \frac{L^2}{\delta^2} \|\vv^*-\vv_0\|^2 \right).
$$
\end{theorem}

The proof is provided in the supplementary material due to its length.
In reality, the constant $\delta>0$ can be very small, probably enough
to make the rate similar to $O(1/k)$. Also the
condition~\eqref{eq:rev.tri} is not easily verifiable without knowing
$\cF(\vw^*)$ a priori. Further, \eqref{eq:rev.tri} implies $\cF(\vw^*)
\ge \cF(\vw_k)$ for $k\ge k_0$, which is not enforced by our
algorithm. Nonetheless, our new result shows that local pointwise
convergence with an accelerated rate $O(1/k^2)$ is possible without
strong convexity, under some special conditions. In our experience,
Algorithm~\ref{alg:1} seemed to exhibit pointwise convergence rate as
fast as, or even faster than, $O(1/k)$, in surprisingly many cases,
even if we chose $\vv'_{k+1} = \vv_{k+1}$ and $\delta=0$: this
motivated us to check both pointwise and ergodic convergence in
Algorithm~\ref{alg:1} for non-strongly convex cases.

\section{DANTZIG SELECTOR WITH THE ORDERED $\ell_1$-NORM}

Here we introduce a new kind of GDS, defined with the ordered
$\ell_1$-norm: for given $p$ parameters $\lambda_1 \ge \cdots \ge
\lambda_p\ge 0$, the Ordered Dantzig Selector (ODS) performs penalized
estimation by solving the problem
\begin{equation}\label{eq:ods}
\text{(ODS)} \quad 
\begin{aligned} \min_{\vw \in\R^p} &\;\; J_\vlam (\vw) := \sum_{i=1}^p \lambda_i \wi{i} \\
\text{s.t.} & \;\;  J_{\vlam}^D (\vX^T(\vy - \vX\vw)) \le 1 
\end{aligned}
\end{equation}
where $\vlam := (\lambda_1,\dots,\lambda_p)$, $\wi{i}$ denotes the
$i$th largest absolute value of the components of the vector
$\vw=(w_1,\dots,w_p)$, and $J_{\vlam}^D$ is the dual norm of
$J_\vlam$. It has been shown that $J_\vlam(\cdot)$ is indeed a
norm~\citep[Proposition 1.2]{BogB15}. Its dual norm has a rather
complicated expression,
$$
 J_{\vlam}^D(\vw) = \max\left\{ \frac{\wi{1}}{\lambda_1}, \cdots, \frac{\sum_{i=1}^p \wi{i}}{\sum_{i=1}^p \lambda_i} \right\} .
$$
Although the ODS~\eqref{eq:ods} can be formulated as a linear program,
it requires exponentially many constraints to express the constraint
set. Our algorithm can avoid handling this thanks to the fact that in
our saddle-point reformulation the dual norm appears in forms of the
double dual norm, i.e., $J_\vlam(\cdot)$:
\begin{align*}
\min_{\vw\in\R^p}\max_{\vv \in \R^p} \, \left\langle \vX^T \begin{bmatrix}\vI & -\vX \end{bmatrix}  \begin{bmatrix} \vy \\ \vw \end{bmatrix}, \vv \right\rangle 
+ J_\vlam(\vw) - J_\vlam(\vv) .
\end{align*}
The proximal operator for $J_\vlam(\cdot)$ can be computed
in 
$O(p\log p)$ time using the stack-based FastProxSL1
algorithm~\citep[Algorithm 4]{BogB15}.

\subsection{False Discovery Rate Control}

In high-dimensional variable selection, some types of statistical
confidence about selection is desired since otherwise the power of
detection of true regressors might be very low or, on the contrary,
the number false discoveries can be too large.

In the popular LASSO approach, variable selection is performed based
on an $\ell_1$-penalized regression,
\begin{equation}\label{eq:lasso}
 \min_{\vw \in \R^p} \;\; \frac{1}{2} \| \vy - \vX \vw \|^2 + \lambda \|\vw\|_1 .
\end{equation}
When observations follow the model $\vy = \vX \vw^* + \vxi$ with
orthogonal design ($\vX^{T}\vX=\vI_p$) and noise $\vxi \sim \cN(0,
\sigma^2 \vI_n)$, one can choose $\lambda \approx \sigma \sqrt{2 \log
  p}$ to control the {\em family-wise error rate} (FWER), the
probability of at least one false rejection. However, this choice is
non-adaptive to data as it does not depend on the sparsity and
magnitude of the true signal, being likely to result in a loss of
power~\citep{BogB15}.

In contrast, in an alternative strategy called the SLOPE which
replaces the $\ell_1$-term in \eqref{eq:lasso} with the ordered
$\ell_1$-norm $J_\vlam(\cdot)$, it has been shown that data-adaptive
{\em false discovery rate (FDR)} control is
possible~\citep{BogB15}. The SLOPE follows the spirit of the
Benjamini-Hochberg correction~\citep{BenH95} in multiple hypothesis
testing, which can adapt to unknown signal sparsity with improved
asymptotic optimality~\citep{AbrF06,BogC11,FroB13,WuZ13}.

Our new proposal, the ODS, shares the same motive as the SLOPE to use
the ordered $\ell_1$-norm, yet in a different context of the Dantzig
Selector. Our next theorem shows that ODS can control FDR, in
orthogonal design cases.

\begin{theorem}\label{thm:ods}
  Under the linear data model $\vy = \vX\vw^* + \vxi$ with $\vX \in
  \R^{n\times p}$, $\vX^T\vX=\vI_p$, and $\vxi \sim \cN(0, \sigma^2
  \vI_n)$, suppose that we choose $\vlam =
  (\lambda_1,\cdots,\lambda_p)$ according to
  $$
     \lambda_i := \sigma \Phi^{-1}\left (1-i \frac{q}{2p}\right )
  $$
  where $\Phi(\cdot)$ denotes the cdf of the standard normal
  distribution.  Then the ODS
  problem~\eqref{eq:ods} has a unique solution $\hat\vw$ with its FDR
  controlled at the level
  $$
    \text{FDR}=\mathbb{E} \bigg [ \frac{V}{\max\{R,1\}} \bigg] \leq q \cdot \frac{p_0}{p}\leq q,
  $$
 $$ \begin{cases}
   p_0 &:= \big| \{ i: w^*_i = 0 \} \big | \;\; \text{(\# true null hypotheses)}\\ 
   V &:=\big |\{i: w^*_i=0, \hat w_i \neq 0\} \big | \;\; \text{(\# false rejections)} \\ 
   R &:=\big |\{i: \hat w_i \neq 0\}\big | \;\; \text{(\# all rejections)}
    \end{cases}$$
\end{theorem}
\begin{proof}
  Our proof is based on showing the equivalence between the ODS and
  the SLOPE estimates under the given conditions, and thereby both
  share the same FDR control. Our full proof is quite technical, and
  is provided in the supplementary.
\end{proof}

For non-orthogonal design, we may need to use a different sequence of
$\lambda_i$'s. For instance, we can consider an adjustment for
Gaussian design cases,
$$
\begin{cases}
  \lambda'_1 &= \lambda_1\\
  \lambda'_i &= \lambda_i \sqrt{ 1+ \frac{\sum_{j < i} (\lambda'_j)^2}{n - i}}, \;\; i \ge 2 ,
\end{cases}
$$
and then for $t = \argmin_{i} \{\lambda'_i\}$, take
\begin{equation}\label{eq:lam.adj}
\begin{aligned}
&  \lambda^G_i = 
\begin{cases} \lambda'_i, & i \le t, \\
  \lambda_{t} & i > t .
\end{cases}
\end{aligned}
\end{equation}
The second step is required to make the sequence $\{\lambda^G_i\}$ to
be non-increasing since otherwise $J_{\vlam^G}(\cdot)$ may not be a
convex function. For details about the adjustment, we refer to
\citep[Section 3.2.2]{BogB15}.

\section{EXPERIMENTS}

We demonstrate our algorithm on the ODS instances with randomly
generated data in various settings. Since the ordered $\ell_1$-norm is
not strongly convex, we run Algorithm~\ref{alg:1} with $\gamma=0$ and
$\vv'_{k+1} = \vv_{k+1}$ unless otherwise specified.

Under the data model $\vy = \vX\vw + \vxi$, we sampled each entry of
the Gaussian design matrix $\vX \in \R^{n\times p}$ and the noise
vector $\vxi$ independently from the normal distribution
$\cN(0,1)$. The true signal $\vw \in \R^p$ was generated to be an
$s$-sparse vector, where the signal strength was set to $w_i = \sqrt{2
  \log p}$ for all nonzero elements $i$. The $\lambda_i$ values were
chosen according to Theorem~\ref{thm:ods} and the
adjustment~\eqref{eq:lam.adj}, with the target FDR level of $q=0.1$.

The performance of Algorithm~\ref{alg:1} (PDSP) has been compared to
the following alternatives:
\begin{itemize}[itemsep=1px,topsep=1pt,leftmargin=2pt]
\item[] {\bf SP Algorithms}:
\begin{itemize}[itemsep=1px,topsep=0pt,leftmargin=1.7em]
\item[$-$] HPE: accelerated hybrid proximal extragradient method~\citep{HeM14}.
\item[$-$] MPL: a variant of the mirror-prox~\citep{Nem04} with
  linesearch~\citep{Tse08}.
\end{itemize}
\item[] {\bf Non-SP Algorithms}:
\begin{itemize}[itemsep=1px,topsep=0pt,leftmargin=1.7em]
\item[$-$] LADMM: linearized ADMM customized for the GDS~\citep{ChaC14}.
\item[$-$] TFOCS: an implementation of Nesterov's
smoothing technique~\citep{BecC11}.
\end{itemize}
\end{itemize}

Unlike the SP algorithms, the non-SP algorithms require to specify
extra parameters difficult to determine: in particular, the penalty
parameter $\rho \ge \|\vX\|^4$ for LADMM and the smoothing parameter
$\mu \approx O(\epsilon)$ for TFOCS. Whenever needed, the values of
$\|\vX\|$ and $\|\vA\| = \|\vX^T\begin{bmatrix} \vI_n &
  -\vX\end{bmatrix}\|$ were estimated by taking inner products of the
matrices with random unit vectors.
For TFOCS, we fixed $\mu = 10^{-2} \gg \epsilon$, since a larger value
than the target optimality $\epsilon$ is usually recommended for better
performance.


All algorithms were stopped if the following condition was satisfied
with an optimality threshold of $\epsilon = 10^{-7}$,
\begin{align*}
\|\vz_k - \vz_{k-1}\| / \max\{1, \|\vz_k\|\} \le \epsilon,
\end{align*}
for either $\vz_k = (\vw_k, \vv_k)$ (pointwise) or $\vz_k = (\overline
\vw_k, \overline \vv_k)$ (ergodic convergence). A tight optimality
threshold is typically required for accurate variable selection.

All experiments were performed on a Linux machine with a quadcore
3.20~GHz Intel Xeon CPU and 24~GB of memory, using MATLAB R2015a.

\subsection{Algorithm Performance}

\begin{table*}[!ht]
\centering
\caption{Algorithm Runtimes (Suboptimality $\epsilon \le 10^{-7}$). Mean (Std) in Seconds over 50 Random ODS Instances.\label{tab:perf}}\medskip
\begin{tabular}{c|cc||cc|cc|cc||cc|cc}\hline
 \multirow{2}{*}{$s$}   & \multirow{2}{*}{$p$}   & \multirow{2}{*}{$n$}
      &  \multicolumn{6}{|c||}{{\bf Saddle-Point Algorithm}} & \multicolumn{4}{|c}{{\bf Non Saddle-Point}} \\ \cline{4-13}
 & &  &  \multicolumn{2}{|c|}{PDSP}  & \multicolumn{2}{|c|}{HPE} & \multicolumn{2}{|c||}{MPL} & \multicolumn{2}{|c|}{LADMM} & \multicolumn{2}{|c}{TFOCS} \\ \hline
\multirow{3}{*}{5} 
&100 & 1000 & 0.04 & (0.03) & 0.05 & (0.03) & 0.20 & (0.12) & 0.10 & (0.22) & 2.92 & (4.88)\\ 
&1000 & 1000 & 1.35 & (3.38) & 48.96 & (339.70) & 3.98 & (6.56) & 15.47 & (29.39) & 54.43 & (291.03)\\ 
&1000 & 100 & 2.79 & (1.63) & 2.28 & (1.57) & 8.15 & (4.05) & 31.87 & (19.99) & 20.02 & (48.73)\\ \hline
\multirow{3}{*}{10} 
&100 & 1000 & 0.19 & (0.40) & 54.22 & (382.27) & 0.74 & (1.30) & 0.41 & (0.53) & 14.33 & (45.28)\\ 
&1000 & 1000 & 2.47 & (6.07) & 1.97 & (3.97) & 6.31 & (11.74) & 29.82 & (31.77) & 37.73 & (85.57)\\ 
&1000 & 100 & 4.99 & (5.61) & 30.05 & (188.19) & 12.93 & (11.34) & 46.78 & (24.39) & 57.27 & (101.36)\\ \hline
\multirow{3}{*}{15} 
&100 & 1000 & 0.33 & (0.68) & 13.95 & (67.75) & 1.07 & (1.49) & 1.32 & (1.70) & 27.56 & (50.32)\\ 
&1000 & 1000 & 3.99 & (8.35) & 2.69 & (5.18) & 9.76 & (15.43) & 39.52 & (32.66) & 38.95 & (103.08)\\ 
&1000 & 100 & 9.88 & (10.70) & 6.93 & (8.00) & 23.86 & (20.82) & 91.52 & (33.56) & 85.77 & (124.23)\\ \hline
\end{tabular}
\end{table*}

The primary advantage of our method (PDSP) is its fast speed with
small runtime variations while being simple to implement. Table~\ref{tab:perf}
compares the runtime of the algorithms over 50 randomly generated ODS
instances in different scenarios, i.e., the combinations of problem
dimensions ($p<n$, $p=n$, $p>n$) and signal sparsity ($s=5, 10, 15$).

Our method has been the most favorable over all cases, except for few
where HPE performed slightly better. However, the HPE algorithm
is far more complicated than ours (see Algorithm 3 and 4 in the
supplementary), having an iterative subproblem solver which requires
to specify extra parameters to control subproblem accuracy.

The advantage of SP methods over non-SP counterparts also looks
apparent. In particular, LADMM, previously proposed for the GDS,
performed well for $p<n$, but quite poorly for the other situations.
Overall, TFOCS has been slower than LADMM. Note that both LADMM and
TFOCS may have performed better if their parameters were tuned for individual
cases: which is exactly what we tried to avoid.

\subsection{FDR Control}

\begin{figure}[!t]
\centering
\includegraphics[height=.18\textheight]{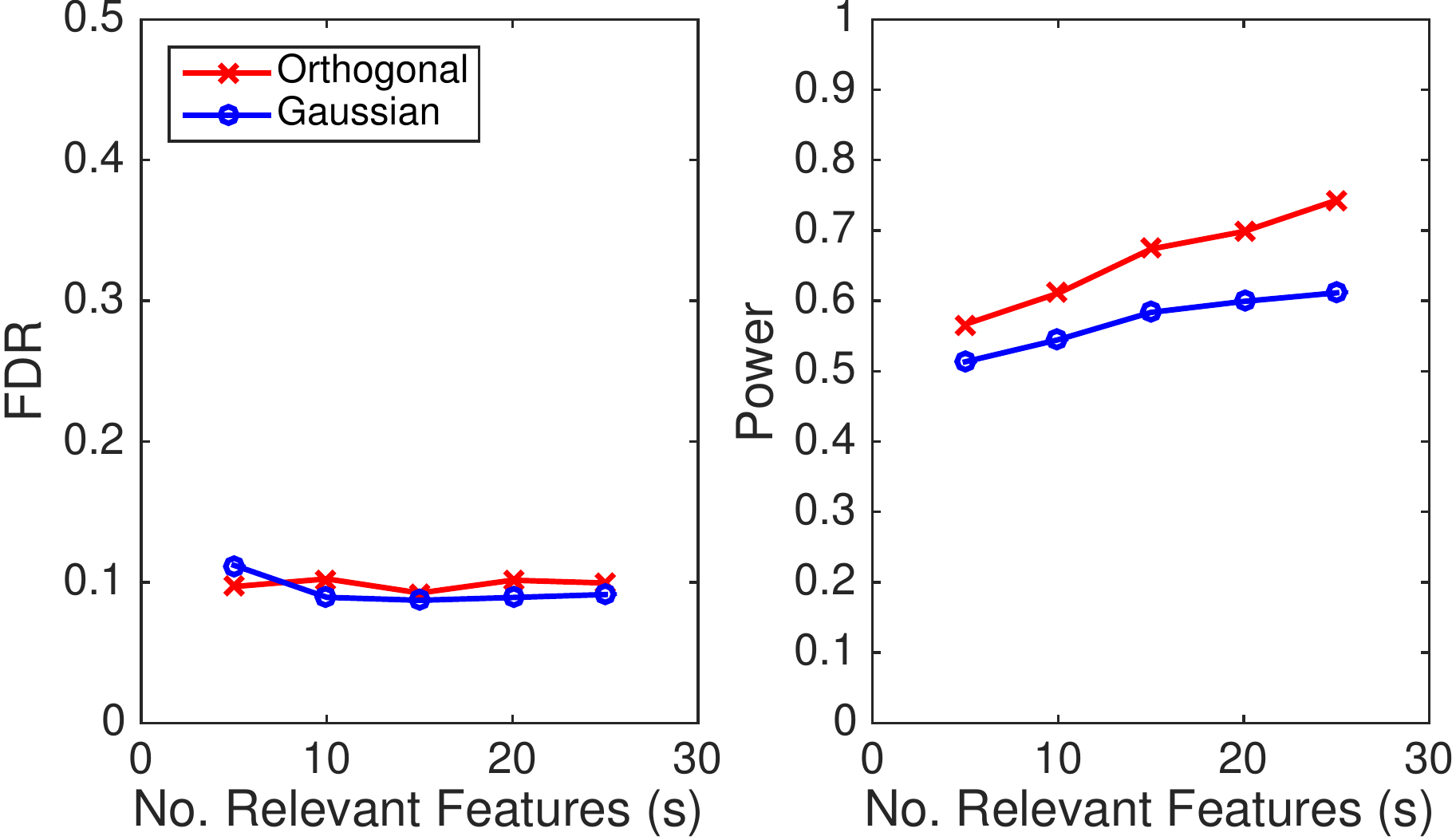}
\caption{
  Mean FDR and Power Detecting Signals of Different Sparsity
  ($s=5,10,15,20,25$). \label{fig:fdr}}
\end{figure}

To show the FDR control property of the ODS (solved with our
algorithm), 
we generated random ODS instances with orthogonal and Gaussian design
matrices of the dimension $n=2000$ and $p=1000$, and compared the FDR
and the power of the two cases for the target FDR level of $q=0.1$.

Figure~\ref{fig:fdr} shows the mean values of these quality criteria
over $300$ repetitions, for increasing numbers of relevant features
($s$) in the true signal (also referred to as signal sparsity). FDR
was indeed controlled at the desired level of $10\%$ in both
orthogonal and Gaussian cases, as we claimed. We observed slightly
improved power with orthogonal design compared to the Gaussian cases:
it is natural since $\lambda$ values were adjusted to control FDR
resulting in larger penalty for the latter.

Comparing to SLOPE using the same $\lambda$ values under Gaussian
design, ODS appeared to be slightly more conservative, improving FDR
and the average number of false discoveries at the cost of a small
decrease in power (data not shown). So ODS would be appealing for
applications like finding biomarkers from high-dimensional genomic
data where false positive discoveries can cost much for follow-up
validation. We leave more precise comparison to SLOPE as future work.


\subsection{Convergence Rate}

\begin{figure}[!t]
\centering
\includegraphics[height=.18\textheight]{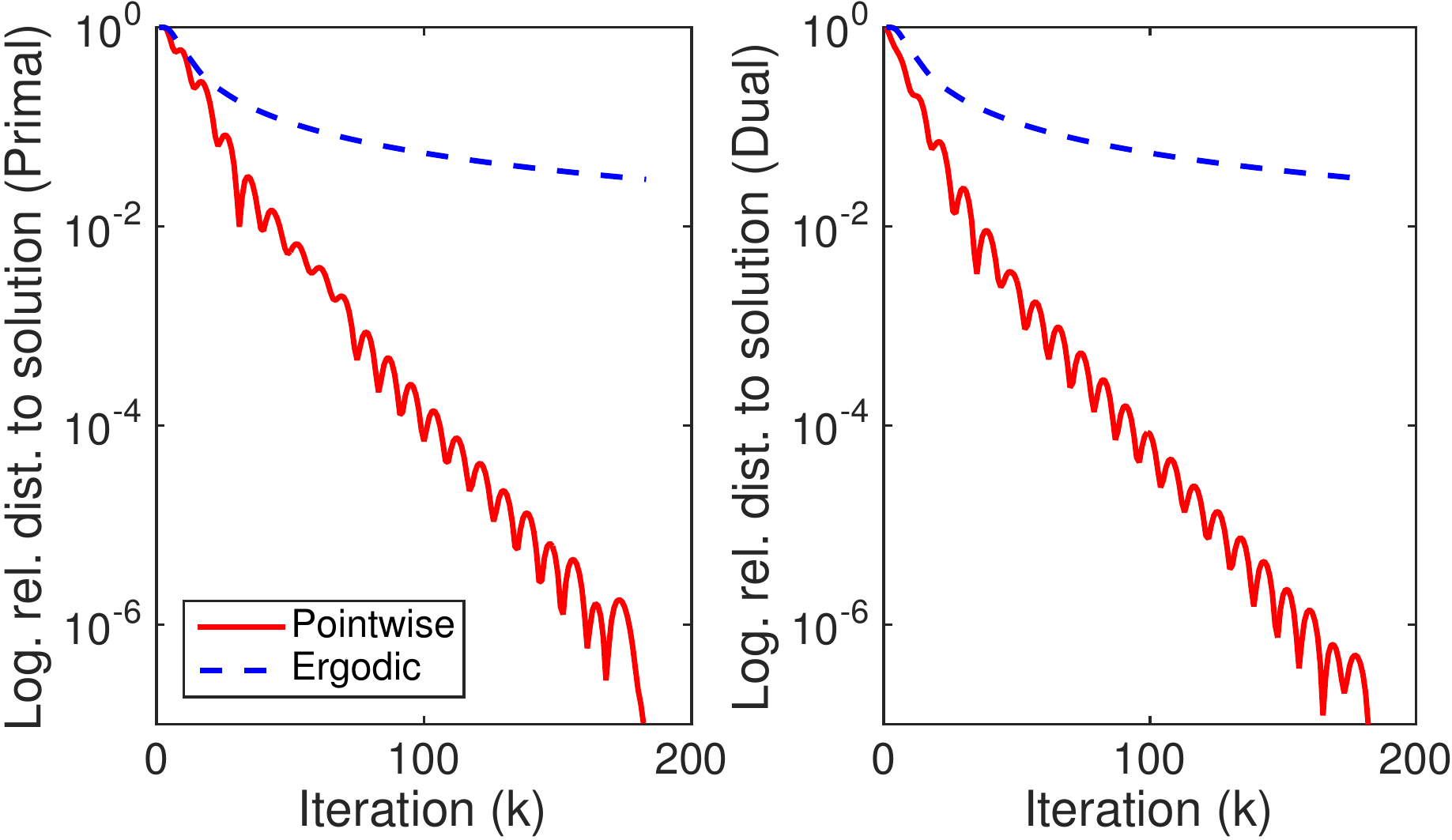}
\caption{Pointwise and Ergodic Convergence. (Left) Primal; (Right)
  Dual.\label{fig:conv}}
\end{figure}

Using our algorithm in experiments, we observed fast pointwise
convergence in almost every case. This was quite surprising, since
pointwise convergence rate is not explained by the existing analysis
in Theorem~\ref{thm:1}, and also expected to be slow, as we discussed
earlier.

Figure~\ref{fig:conv} shows one instance of the randomly generated
Gaussian design cases with $p=n=1000$, $s=15$, and $q=0.1$ (behavior
was quite similar in other settings). We ran our algorithm twice for
the same data, 1) to obtain the primal and the dual solutions, then 2)
to obtain the relative distances of iterates to their corresponding
solution, such as $\|\vw_k - \vw^*\| / \|\vw^*\|$.

As we can see, the averaged iterate (denoted by ``Ergodic'') showed
the expected $O(1/k)$ convergence rate. In contrast, the non-averaged
iterates (``pointwise'') converged much faster, even exhibiting
typical fluctuation patterns of accelerated gradient method. We
believe that this behavior is closely related to the local strong
convexity and acceleration we discussed.

In fact, in Figure~\ref{fig:conv}, neither any information of local
strong convexity nor the alternative choice of $\vv'_{k+1} =
2\vv_{k+1}$ was used. When the latter option was used, our algorithm
showed even faster pointwise convergence, but there were some cases
the algorithm did not converge, which would be when the required local
strong convexity assumption was not satisfied.

\subsection{Local Strong Convexity}

\begin{figure}[!t]
\centering
\includegraphics[width=.95\linewidth]{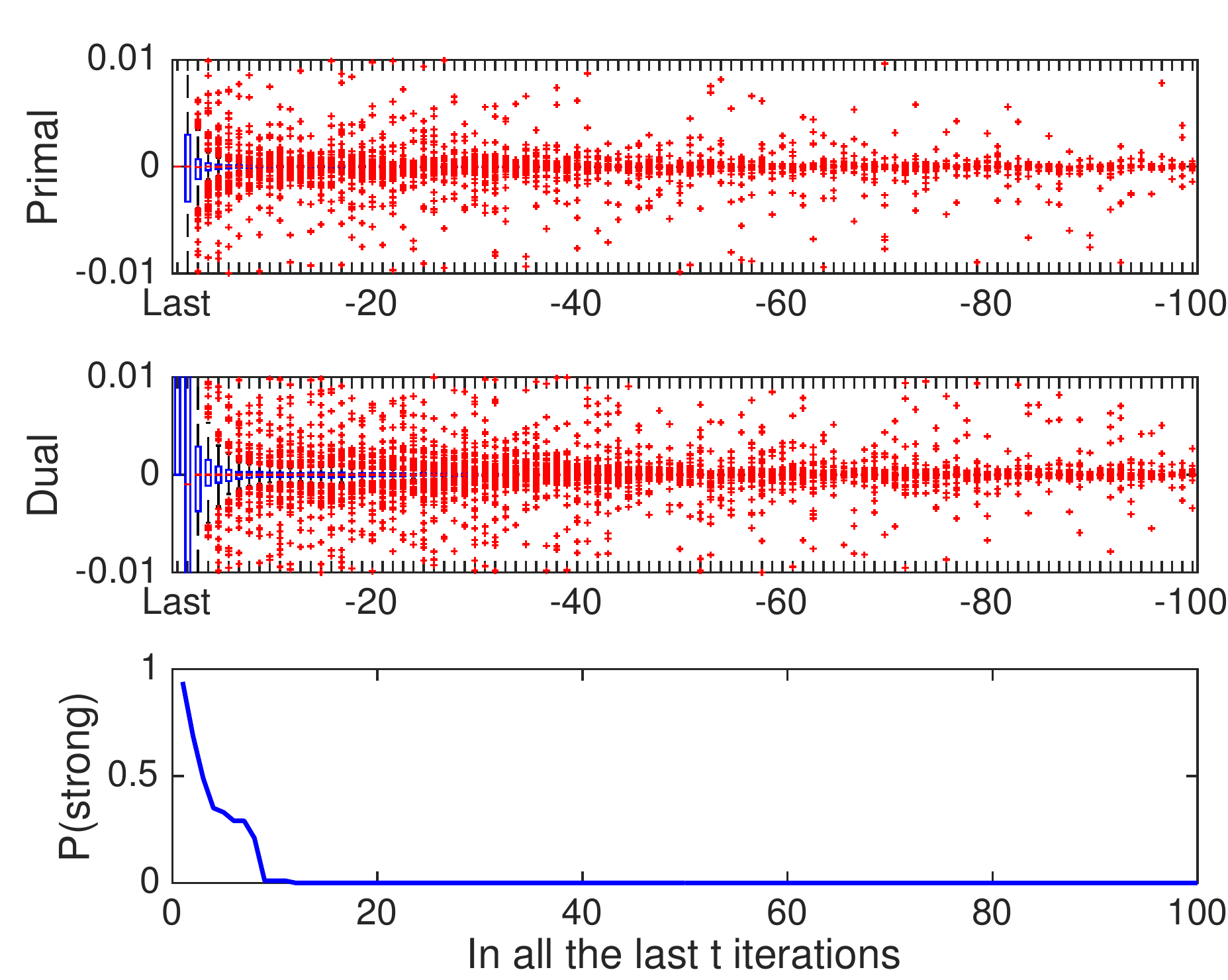}
\caption{(Top; Middle) Local Strong Convexity Estimates for
  Primal:~$\cF$ and Dual:~$\cG$ in the last 100 iterations.  (Bottom)
  Probability of Local Strong Convexity in Primal or in Dual, for all
  of the last $t$ iterations.\label{fig:strconv}}
\end{figure}

We again generated 300 random ODS instances with the same settings as
in the previous experiment, to simulate how often the local strong
convexity condition~\eqref{eq:strong.conv} would be fulfilled, and to
what degree.

Figure~\ref{fig:strconv} (top and middle) reports the box-plots of the
local strong convexity estimates:
$$
  \frac{\cF(\vw^*)-\cF(\vw_k)-\frac 12 \langle \vg, \vw^*-\vw_k\rangle}{\|\vw^*-\vw_k\|^2}, \;\; \vg \in \partial \cF(\vw_k),
$$
in the primal, and equivalent quantities regarding $\cG$ in the dual,
evaluated for the last 100 iterations of each run. As we approached
the last iteration, these values varied more away from zero, where the
chance of being positive was nearly $50\%$ in the primal and dual,
resp. In fact, for acceleration to happen, it is very likely from
Theorem~\ref{thm:2} that the values need to be positive in either
primal or dual: Figure~\ref{fig:strconv} (bottom) shows the chance of
such events to happen, in {\em all} of the last $t$ iterations: the
probability seemed to approach one as $t\to 1$. This indicates that
local acceleration near an optimal solution would be highly likely.

\section{CONCLUSION}

We proposed PDSP, a fast and simple primal-dual 
algorithm to solve the saddle-point formulation of the generalized
Dantzig selector. While achieving the known optimal convergence rate,
we showed that our algorithm can exhibit a faster rate, taking the
advantage of local acceleration. We also introduced the ordered
Dantzig selector with FDR control, a new instance of the GDS, which we
hope will foster further research in variable selection and signal
recovery.

\subsubsection*{Acknowledgements}
SL was supported by Deutsche Forschungsgemeinschaft (DFG) within the
Collaborative Research Center SFB 876, project C1. 
D.B. and M.B. were supported by European Union's 7th Framework
Programme under Grant Agreement no 602552 and by the Polish Ministry
of Science and Higher Education under grant agreement
2932/7.PR/2013/2. 

\clearpage
\bibliography{sp_dantzig}
\bibliographystyle{abbrvnat}


\iffullver

\clearpage

\section*{Appendix}

\subsection*{Proof of Theorem~\ref{thm:1}}

This proof essentially follows that of \citet[Theorem 1]{ChaP11}, with
a small trick to reformulate the GDS-SP to the formulation discussed
in the original proof. 

From the definition of the proximal steps in Algorithm~\ref{alg:1},
\begin{align*}
 \vv_{k+1} &= (I + \sigma \partial \cG)^{-1} (\vv_k + \sigma (\vX^T\vy - \vX^T\vX \vw'_k) ) \\
 \vw_{k+1} &= (I + \tau \partial \cF)^{-1} (\vw_k + \tau \vX^T\vX \vv'_{k+1}) ,
\end{align*}
it is implied that
\begin{equation}\label{eq:subg}
\begin{aligned}
 \partial \cG(\vv_{k+1}) &\ni \frac{\vv_k - \vv_{k+1}}{\sigma} + \vX^T\vy - \vX^T\vX \vw'_k\\
 \partial \cF(\vw_{k+1}) &\ni \frac{\vw_k - \vw_{k+1}}{\tau} + \vX^T\vX \vv'_{k+1} .
\end{aligned}
\end{equation}
Since $\cG$ and $\cH$ are convex functions, it follows for any $(\vw,\vv)$,
\begin{align*}
 \cG(\vv) &\ge \cG(\vv_{k+1}) + \sigma^{-1} \langle \vv_k - \vv_{k+1}, \vv - \vv_{k+1} \rangle \\
 & \quad + \langle \vX^T\vy - \vX^T\vX\vw'_k, \vv - \vv_{k+1} \rangle \\
 \cF(\vw) &\ge \cF(\vw_{k+1}) + \tau^{-1} \langle \vw_k - \vw_{k+1}, \vw - \vw_{k+1} \rangle \\
 &\quad + \langle \vX^T\vX(\vw-\vw_{k+1}), \vv'_{k+1} \rangle .
\end{align*}
We define augmentations of $\vw$'s with $\vy$, e.g. $\vz_k := [\vy;
\vw_k] \in \R^{n+p}$, and $\cH(\vz) := \cH(\vy,\vw) = \cF(\vw)$. Then
the preceding inequalities can be rewritten as follows,
\begin{equation}\label{eq:pf.0}
\begin{aligned}
 \cG(\vv) &\ge \cG(\vv_{k+1}) + \sigma^{-1} \langle \vv_k - \vv_{k+1}, \vv - \vv_{k+1} \rangle \\
 & \quad + \langle \vA\vz'_k, \vv - \vv_{k+1} \rangle \\
 \cH(\vz) &\ge \cH(\vz_{k+1}) + \tau^{-1} \langle \vz_k - \vz_{k+1}, \vz - \vz_{k+1} \rangle \\
 &\quad - \langle \vA(\vz-\vz_{k+1}), \vv'_{k+1} \rangle .
\end{aligned}
\end{equation}
Summing both inequalities and using an elementary result $\langle
a-c,b-c\rangle = \|a-c\|^2/2 + \|b-c\|^2/2 - \|a-b\|^2/2$, we have
\begin{equation}\label{eq:pf.1}
\begin{aligned}
& \frac{\|\vv-\vv_k\|^2}{2\sigma} + \frac{\|\vz-\vz_k\|^2}{2\tau} \\
& \; \ge [ \langle \vA\vz_{k+1}, \vv \rangle - \cG(\vv) + \cH(\vz_{k+1})]\\
& \quad - [\langle \vA\vz, \vv_{k+1} \rangle - \cG(\vv_{k+1}) + \cH(\vz) ]\\
& \quad + \frac{\|\vv-\vv_{k+1}\|^2}{2\sigma} + \frac{\|\vz-\vz_{k+1}\|^2}{2\tau} \\
& \quad + \frac{\|\vv_k - \vv_{k+1}\|^2}{2\sigma} + \frac{\|\vz_k - \vz_{k+1} \|^2}{2\tau} \\
& \quad + \langle \vA(\vz_{k+1} - \vz'_k), \vv_{k+1} - \vv \rangle\\
& \quad - \langle \vA(\vz_{k+1} - \vz), \vv_{k+1} - \vv'_{k+1} \rangle .
\end{aligned}
\end{equation}
Replacing the extragradient steps,
$$
\begin{cases}
 \vz'_k &= \vz_k + \theta (\vz_k - \vz_{k-1}) \\
 \vv'_{k+1} &= \vv_{k+1} ,
\end{cases}
$$
the last two lines of \eqref{eq:pf.1} can be bounded as follows,
\begin{align*}
  & \langle \vA(\vz_{k+1} - \vz'_{k+1}, \vv_{k+1} - \vv \rangle \\
  & \; \ge \langle \vA(\vz_{k+1} - \vz_k), \vv_{k+1}- \vv \rangle \\
  & \quad - \theta \langle \vA(\vz_k - \vz_{k-1})), \vv_k- \vv \rangle \\
  & \quad - \theta L \|\vz_k - \vz_{k-1}\| \|\vv_{k+1}- \vv_k \| \\
  & \; \ge \langle \vA(\vz_{k+1} - \vz_k), \vv_{k+1}- \vv \rangle \\
  & \quad - \theta \langle \vA(\vz_k - \vz_{k-1})), \vv_k- \vv \rangle \\
  & \quad - \theta^2 \sqrt{\sigma\tau} L \frac{\|\vz_k -
    \vz_{k-1}\|^2}{2\tau} - \sqrt{\sigma\tau} L \frac{\|\vv_{k+1}-
    \vv_k \|^2}{2\sigma}\end{align*} where the first inequality is due
to Cauchy-Schwarz and $L := \|\vA\|$, and the second is using $|ab|
\le (\alpha/2) a^2 + 1/(2\alpha) b^2$ for any $\alpha>0$ (in this case
$\alpha = \sqrt{\sigma/ \tau}$).

Combining this with the full inequality~\eqref{eq:pf.1}, we have
\begin{align*}
& \frac{\|\vv-\vv_k\|^2}{2\sigma} + \frac{\|\vz-\vz_k\|^2}{2\tau} \\
& \; \ge [ \langle \vA\vz_{k+1}, \vv \rangle - \cG(\vv) + \cH(\vz_{k+1})]\\
& \quad - [\langle \vA\vz, \vv_{k+1} \rangle - \cG(\vv_{k+1}) + \cH(\vz) ]\\
& \quad + \frac{\|\vv-\vv_{k+1}\|^2}{2\sigma} 
 + \frac{\|\vz-\vz_{k+1}\|^2}{2\tau} \\
& \quad \left( 1 - \sqrt{\sigma\tau} L \right) \frac{\|\vv_{k+1} - \vv_k\|^2}{2\sigma} \\
& \quad + \frac{\|\vz_k - \vz_{k+1} \|^2}{2\tau} 
 - \theta^2\sqrt{\sigma\tau} L \frac{\|\vz_k - \vz_{k-1}\|^2}{2\tau} \\
& \quad +  \langle \vA(\vz_{k+1} - \vz_k), \vv_{k+1}- \vv \rangle \\
& \quad - \theta \langle \vA(\vz_k - \vz_{k-1})), \vv_k- \vv \rangle \\
\end{align*}

If we choose parameters so that $\theta = 1$ and $\sigma\tau L^2 \le 1$, 
and define 
$$
\Delta_k := \frac{\|\vv-\vv_k\|^2}{2\sigma} + \frac{\|\vz-\vz_k\|^2}{2\tau},
$$
it follows that
\begin{align*}
 \Delta_k
& \; \ge [ \langle \vA\vz_{k+1}, \vv \rangle - \cG(\vv) + \cH(\vz_{k+1})]\\
& \quad - [ \langle \vA\vz, \vv_{k+1} \rangle - \cG(\vv_{k+1}) + \cH(\vz)] \\
& \quad + \Delta_{k+1} \\
& \quad + \left( 1 - \sqrt{\sigma\tau} L \right) \frac{\|\vv_{k+1} - \vv_k\|^2}{2\sigma} \\
& \quad + \frac{\|\vz_k - \vz_{k+1} \|^2}{2\tau} - \sqrt{\sigma\tau} L \frac{\|\vz_k - \vz_{k-1}\|^2}{2\tau} \\
& \quad +  \langle \vA(\vz_{k+1} - \vz_k), \vv_{k+1}- \vv \rangle \\
& \quad -  \langle \vA(\vz_k - \vz_{k-1})), \vv_k- \vv \rangle .
\end{align*}

Summing up the above inequality for $k=0$ to $t-1$, with
$\vz_{-1}=\vz_0$ and $\vv_{-1} = \vv_0$, we get
\begin{align*}
 \Delta_0
 \ge & \sum_{k=1}^t \Big\{ [\langle \vA\vz_k, \vv \rangle - \cG(\vv) + \cH(\vz_k) ]\\
& \quad - [\langle \vA\vz, \vv_k \rangle - \cG(\vv_k) + \cH(\vz) ] \Big\}\\
& \;\; + \Delta_t \\
& \quad + \left( 1 - \sqrt{\sigma\tau} L \right) \sum_{k=1}^t \frac{\|\vv_k - \vv_{k-1}\|^2}{2\sigma} \\
& \quad + \left( 1 - \sqrt{\sigma\tau} L \right) \sum_{k=1}^{t-1} \frac{\|\vz_k - \vz_{k-1} \|^2}{2\tau} + \frac{\|\vz_t - \vz_{t-1}\|^2}{2\tau} \\
& \;\; + \langle \vA (\vz_t - \vz_{t-1}), \vv_t - \vv \rangle \\
 \ge & \sum_{k=1}^t \Big\{ [\langle \vA\vz_k, \vv \rangle - \cG(\vv) + \cH(\vz_k) ]\\
& \quad - [\langle \vA\vz, \vv_k \rangle - \cG(\vv_k) + \cH(\vz) ] \Big\}\\
& \;\; + \Delta_t \\
& \quad + \left( 1 - \sqrt{\sigma\tau} L \right) \sum_{k=1}^t \frac{\|\vv_k - \vv_{k-1}\|^2}{2\sigma} \\
& \quad + \left( 1 - \sqrt{\sigma\tau} L \right) \sum_{k=1}^{t-1} \frac{\|\vz_k - \vz_{k-1} \|^2}{2\tau} + \frac{\|\vz_t - \vz_{t-1}\|^2}{2\tau} \\
& \;\; - \frac{\| \vz_t - \vz_{t-1} \|^2}{2\tau} - \tau L^2 \frac{\|\vv_t - \vv\|^2}{2} .
\end{align*}

That is, for any $(\vz,\vv)$,
\begin{equation}\label{eq:pf.3}
\begin{aligned}
&  (1 - \tau\sigma L^2) \frac{\|\vv-\vv_t\|^2}{\sigma} + \frac{\|\vz-\vz_t\|^2}{\tau}  \\
& \quad + 2 \sum_{k=1}^t \Big\{ [ \langle \vA\vz_k, \vv \rangle - \cG(\vv) + \cH(\vz_k)]\\
& \quad - [\langle \vA\vz, \vv_k \rangle - \cG(\vv_k) + \cH(\vz) ] \Big\}\\
& \quad + \left( 1 - \sqrt{\sigma\tau} L \right) \sum_{k=1}^t \frac{\|\vv_k - \vv_{k-1}\|^2}{\sigma} \\
& \quad + \left( 1 - \sqrt{\sigma\tau} L \right) \sum_{k=1}^{t-1} \frac{\|\vz_k - \vz_{k-1} \|^2}{\tau} \\
& \le \frac{\|\vv-\vv_0\|^2}{\sigma} + \frac{\|\vz-\vz_0\|^2}{\tau}  .
\end{aligned}
\end{equation}

For any saddle point $(\vz^*,\vv^*)$ satisfying the conditions in
\eqref{eq:sp}, we observe that
\begin{equation}\label{eq:pf.saddle}
\begin{aligned}
& [\langle \vA \vz_k, \vv^* \rangle - \cG(\vv^*) + \cH(\vz_k)] \\
& \;\; \ge [\langle \vA \vz^*, \vv^* \rangle - \cG(\vv^*) + \cH(\vz^*)] \\
& \;\; \ge [\langle \vA \vz^*, \vv_k \rangle - \cG(\vv_k) + \cH(\vz^*)] .
\end{aligned}
\end{equation}
That is, for $(\vz,\vv) = (\vz^*,\vv^*)$ the first summation in
\eqref{eq:pf.3} is non-negative, and therefore $(\vz_k,\vv_k)$ is
bounded, showing the property (a).

Also, from \eqref{eq:pf.3} due the convexity of $\cG$ and $\cH$, we
have the following result for $\overline \vz_t = \frac{1}{t}
\sum_{k=1}^t \vz_k$ and $\overline \vv_t = \frac{1}{t}\sum_{k=1}^t
\vv_k$,
\begin{equation}\label{eq:pf.4}
\begin{aligned}
& [ \langle \vA \overline \vz_t, \vv \rangle - \cG(\vv) + \cH(\overline \vz_t) ] - [\langle \vA\vz, \overline \vv_t \rangle - \cG(\overline \vv_t) + \cH(\vz) ] \\
& \;\; \le \frac{1}{t} \left( \frac{\|\vv - \vv_0\|^2}{2\sigma} + \frac{\|\vz - \vz_0\|^2}{2\tau} \right) \\
& \;\; \le \frac{1}{t} \left( \frac{\|\vv - \vv^* \|^2}{2\sigma} + \frac{\|\vv^* - \vv_0\|^2}{2\sigma} \right .\\ 
& \quad  + \left. \frac{\|\vz - \vz^*\|^2}{2\tau} + \frac{\|\vz^* - \vz_0\|^2}{2\tau} \right) \\
& \;\; \le \frac{1+C}{t} \left(  \frac{\|\vv^* - \vv_0\|^2}{2\sigma} +  \frac{\|\vz^* - \vz_0\|^2}{2\tau} \right) ,
\end{aligned}
\end{equation}
where $C = 1/(1-\tau\sigma L^2)$ and the last inequality is due to
the previous result (a). This shows the first part of (b).

Furthermore, with $t \to \infty$, \eqref{eq:pf.4} implies for any
limit point $(\hat \vz, \hat \vv)$ of $(\overline \vz_t, \overline
\vv_t)$ that
$$
  [ \langle \vA \hat \vz, \vv \rangle - \cG(\vv) + \cH(\hat \vz) ] - [\langle \vA\vz, \hat \vv \rangle - \cG(\hat \vv) + \cH(\vz) ] \le 0
$$
since $\cH$ and $\cG$ are l.s.c., i.e.
$$
 \liminf_{t \to \infty} \cH (\overline \vz_t) \ge \cH (\hat \vz), \;\; \liminf_{t\to \infty} \cG(\overline \vv_t) \ge \cG(\hat \vv) .
$$
This implies that $(\hat \vz, \hat \vv)$ is a saddle-point of \eqref{eq:gds.sp}, the second claim in (b).

Finally, since $(1-\sqrt{\sigma\tau}L) \ge 0$, \eqref{eq:pf.3} also
implies that $\|\vv_k-\vv_{k-1}\| \to 0$ and $\|\vz_k-\vz_{k-1}\| \to
0$ as $k\to\infty$, and therefore $(\vz_k,\vv_k)$ should have a limit,
say $(\tilde \vz, \tilde \vv)$. Then again from \eqref{eq:pf.3}, we have at the limit
$$
 [ \langle \vA \tilde z, \vv^* \rangle - \cG(\vv^*) + \cH(\tilde z)]
 - [\langle \vA\vz^*, \tilde \vv \rangle - \cG(\tilde \vv) + \cH(\vz^*) ] = 0 ,
$$
which tells that $(\tilde \vz, \tilde \vv)$ is also a saddle-point, showing the last claim (c). \qed

\subsection*{Proof of Theorem~\ref{thm:2}}

Hence $\cF$ satisfies the local strong
convexity~\eqref{eq:strong.conv}, so does the function $\cH(\vz) =
\cH(\vy,\vw) = \cF(\vw)$ on augmented vectors $\vz = (\vy,\vw)$, and
from the expression of subgradients~\eqref{eq:subg} we have (with
$\tau'_k = 2\tau_k$)
\begin{align*}
 \cH(\vz^*) &\ge \cH(\vz_{k+1}) + \left\langle \frac{\vz_k - \vz_{k+1}}{\tau'_k}, \vz^* - \vz_{k+1} \right\rangle \\
& - \left\langle \vA (\vz^*-\vz_{k+1}), \frac{\vv'_{k+1}}{2} \right\rangle + \frac{\delta}{2} \|\vz^* - \vz_{k+1}\|^2.
\end{align*}
Also, from the convexity of $\cG$~\eqref{eq:pf.0},
\begin{align*}
  \cG(\vv^*) &\ge \cG(\vv_{k+1}) + \left\langle \frac{\vv_k - \vv_{k+1}}{\sigma_k}, \vv^* - \vv_{k+1} \right\rangle \\
 & \quad + \langle \vA\vz'_k, \vv^* - \vv_{k+1} \rangle .
\end{align*}
With these, the previous inequality~\eqref{eq:pf.1} modifies as
follows,
\begin{align*}
& \frac{\|\vv^*-\vv_k\|^2}{2\sigma_k} + \frac{\|\vz^*-\vz_k\|^2}{2\tau'_k} \\
& \; \ge [ \langle \vA\vz_{k+1}, \vv^* \rangle - \cG(\vv^*) + \cH(\vz_{k+1})]\\
& \quad - [\langle \vA\vz^*, \vv_{k+1} \rangle - \cG(\vv_{k+1}) + \cH(\vz^*) ]\\
& \quad + \frac{\delta}{2} \|\vz^* - \vz_{k+1}\|^2 + \frac{\|\vv^*-\vv_{k+1}\|^2}{2\sigma_k} + \frac{\|\vz^*-\vz_{k+1}\|^2}{2\tau'_k} \\
& \quad + \frac{\|\vv_k - \vv_{k+1}\|^2}{2\sigma_k} + \frac{\|\vz_k - \vz_{k+1} \|^2}{2\tau'_k} \\
& \quad + \langle \vA(\vz_{k+1} - \vz'_k), \vv_{k+1} - \vv^* \rangle\\
& \quad - \langle \vA(\vz_{k+1} - \vz^*), \vv_{k+1} - \frac{\vv'_{k+1}}{2} \rangle .
\end{align*}
From \eqref{eq:pf.saddle}, the expression in the first two terms of
the right-hand side is bounded below by zero.

Choosing the extragradient steps as follows,
$$
\begin{cases}
 \vz'_k &= \vz_k + \theta_{k-1} (\vz_k - \vz_{k-1}) \\
 \vv'_{k+1} &= 2 \vv_{k+1} ,
\end{cases}
$$
it follows that
\begin{align*}
& \frac{\|\vv^*-\vv_k\|^2}{2\sigma_k} + \frac{\|\vz^*-\vz_k\|^2}{2\tau'_k} \\
& \; \ge \frac{\delta}{2} \|\vz^* - \vz_{k+1}\|^2 + \frac{\|\vv^*-\vv_{k+1}\|^2}{2\sigma_k} + \frac{\|\vz^*-\vz_{k+1}\|^2}{2\tau'_k} \\
& \quad + \frac{\|\vv_k - \vv_{k+1}\|^2}{2\sigma_k} + \frac{\|\vz_k - \vz_{k+1} \|^2}{2\tau'_k} \\
& \quad + \langle \vA(\vz_{k+1} - \vz_k), \vv_{k+1} - \vv^* \rangle\\
& \quad - \theta_{k-1} \langle \vA(\vz_{k} - \vz_{k-1}), \vv_k - \vv^* \rangle\\
& \quad - \theta_{k-1} L \| \vz_k - \vz_{k-1} \| \| \vv_{k+1} - \vv_k\| .
\end{align*}

The rest of the proof is very similar to that of Section 5 of
\citet{ChaP11}, from eq. (39) to (42) therein. From the above, we
have
\begin{align*}
& \frac{\|\vv^*-\vv_k\|^2}{\sigma_k} + \frac{\|\vz^*-\vz_k\|^2}{\tau'_k} \\
& \; \ge (1+\delta\tau'_k)\frac{\tau'_{k+1}}{\tau'_k} \frac{\|\vz^* - \vz_{k+1}\|^2}{\tau'_{k+1}} + \frac{\sigma_{k+1}}{\sigma_k} \frac{\|\vv^*-\vv_{k+1}\|^2}{\sigma_{k+1}} \\
& \quad + \frac{\|\vv_k - \vv_{k+1}\|^2}{\sigma_k} + \frac{\|\vz_k - \vz_{k+1} \|^2}{\tau'_k} - \frac{\|\vv_{k+1}-\vv_k\|^2}{\sigma_k}\\
& \quad + 2\langle \vA(\vz_{k+1} - \vz_k), \vv_{k+1} - \vv^* \rangle\\
& \quad - 2\theta_{k-1} \langle \vA(\vz_{k} - \vz_{k-1}), \vv_k - \vv^* \rangle\\
& \quad - \theta_{k-1}^2 L^2 \sigma_k \tau'_{k-1}\frac{\| \vz_k - \vz_{k-1} \|^2}{\tau'_{k-1}} .
\end{align*}
We choose sequences $\tau'_k$ and $\sigma_k$ such that
$$
 (1+\delta\tau'_k) \frac{\tau'_{k+1}}{\tau'_k} = \frac{\sigma_{k+1}}{\sigma_k} = \frac{1}{\theta_k} = \frac{\tau'_k}{\tau'_{k+1}} > 1 .
$$
Denoting 
$$
 \Delta_k := \frac{\|\vv^* - \vv_k\|^2}{\sigma_k} + \frac{\|\vz^*-\vz_k\|^2}{\tau'_k} ,
$$
dividing the both sides of the above inequality by $\tau'_k$, and
using $L^2 \sigma_k \tau'_k = L^2 \sigma_0 \tau'_0 \le 1$, we get
\begin{align*}
& \frac{\Delta_k}{\tau'_k} 
 \ge \frac{\Delta_{k+1}}{\tau'_{k+1}} 
 + \frac{\|\vz_k - \vz_{k+1} \|^2}{(\tau'_k)^2} - \frac{\| \vz_k - \vz_{k-1} \|^2}{(\tau'_{k-1})^2}  \\
& \quad + \frac{2}{\tau'_k} \langle \vA(\vz_{k+1} - \vz_k), \vv_{k+1} - \vv^* \rangle\\
& \quad - \frac{2}{\tau'_{k-1}} \langle \vA(\vz_{k} - \vz_{k-1}), \vv_k - \vv^* \rangle .
\end{align*}
Applying this inequality for $k=0,\cdots,(t-1)$, and using $\vz_{-1} = \vz_0$, it leads to
\begin{align*}
 \frac{\Delta_0}{\tau'_0} 
& \ge \frac{\Delta_t}{\tau'_t}  + \frac{\|\vz_{t-1} - \vz_t \|^2}{(\tau'_{t-1})^2} \\
& \quad + \frac{2}{\tau'_{t-1}} \langle \vA(\vz_t - \vz_{t-1}), \vv_t - \vv^* \rangle\\
& \ge \frac{\Delta_t}{\tau'_t}  + \frac{\|\vz_{t-1} - \vz_t \|^2}{(\tau'_{t-1})^2} \\
& \quad - \frac{\|\vz_t - \vz_{t-1}\|^2}{(\tau'_{t-1})^2} - L^2 \|\vv_t - \vv^* \|^2 .
\end{align*}
Rearranging terms, using $L^2 \sigma_k \tau'_k = L^2 \sigma_0 \tau'_0
$, and replacing $\tau'_k = 2\tau_k$, we finally obtain
\begin{align*}
& 4 \tau_t^2 \frac{1-2 L^2\sigma_0\tau_0}{2\sigma_0\tau_0} \|\vv^* - \vv_t\|^2 + \|\vz^* - \vz_t\|^2\\
& \le 4 \tau_t^2 \left( \frac{\|\vz^*-\vz_0\|^2}{ 4\tau_0^2} + \frac{\|\vv^*-\vv_0\|^2}{2\sigma_0\tau_0} \right) .
\end{align*}
If we choose $\sigma_0, \tau_0$ so that $2 \sigma_0\tau_0 L^2 =1$, then
\begin{align*}
& \|\vz^* - \vz_t\|^2 \le 4 \tau_t^2 \left( \frac{\|\vz^*-\vz_0\|^2}{ 4\tau_0^2} + L^2 \|\vv^*-\vv_0\|^2 \right) .
\end{align*}
The proof completes if we show $\tau_t \sim t^{-1}$ for all $t \ge
t_0$. Note that our choice of $\tau_k$ satisfies (with $\tau'_k = 2\tau_k$),
$$
 (1+2 \delta \tau_k) \frac{\tau_{k+1}}{\tau_k} = \frac{\tau_k}{\tau_{k+1}},
$$
which is an identical choice to that of Lemma 1 of \citet{ChaP11},
replacing $\gamma = \delta$ therein. Therefore Corollary 1 of
\citet{ChaP11} applies as it is, which shows that $\lim_{t\to\infty} t
\delta \tau_t = 1$. \qed

\subsection*{FDR Control of the ODS}

\subsubsection*{Properties of $J_\vlam$ and $J_\vlam^D$}



\begin{proposition} \label{Props1} If $\va$, $\vb \in \mathbb{R}^p$
  are such that $|\va| \preceq |\vb|$, then the vectors sorted in
  decreasing component magnitudes satisfy $|\va|_{(\cdot)} \preceq
  |\vb|_{(\cdot)}$.
\end{proposition}

\begin{proof}
  Without loss of generality assume that $\va$ and $\vb$ are
  nonnegative and that $a_1\geq \ldots\geq a_p$. We will show that
  $a_k \leq b_{(k)}$ for $k\in \{1,\ldots,p\}$. Fix such $k$ and
  consider the set $S_k:=\{b_i:\ b_i \geq a_k\}$. It is enough to show
  that $|S_k|\geq k$, where $|S_k|$ is the number of elements in
  $S_k$. For each $j\in \{1,\ldots,k\}$ we have
\begin{equation*}
b_j \geq a_j \geq a_k\ \Longrightarrow\ b_j\in S_k,
\end{equation*} 
what proves the last statement.
\end{proof}

\begin{corollary}
\label{17021046}
Let $\va, \vb \in \mathbb{R}^p$ and $|\va| \preceq |\vb|$. Then
${J_\vlam (\va) \leq J_\vlam (\vb)}$ since $J_\vlam
(\va)=\vlam^T|\va|_{(\cdot)} \leq \vlam^T|\vb|_{(\cdot)}= J_\vlam
(\vb)$.
\end{corollary}

\begin{proposition}[Pulling to zero] \label{Props4} For fixed sequence
  $\lambda_1\geq\ldots\geq\lambda_p\geq0$, let $\vw\in \mathbb{R}^p$ be
  such that $w_j>0$ for some ${j\in \{1,\ldots,p\}}$. For
  $\varepsilon\in (0,w_j]$ define:
$$ (w_\varepsilon)_i=\left\{\begin{array}{ll} w_j-\varepsilon, & i=j\\w_i,& \textrm{otherwise} \end{array} \right.. $$
Then:
\newline
(i) $J_\lambda(\vw_\varepsilon) \leq J_\lambda(\vw)$,
\newline
(ii) if $\vlam \succ  0$, then $J_\lambda(\vw_\varepsilon) < J_\lambda(\vw).$
\end{proposition}
\begin{proof}
  Let $\pi:\{1,\ldots,p\}\longrightarrow\{1,\ldots,p\}$ be permutation
  such as $\sum_{i=1}^p\lambda_i
  (w_\varepsilon)_{(i)}=\sum_{i=1}^p\lambda_{\pi(i)}
  (w_\varepsilon)_i$ for each $i$ in $\{1,\ldots,p\}$. Using the
  rearrangement inequality~\citep{HarL94}
\begin{align*}
& J_\vlam(\vw)-J_\vlam(\vw_\varepsilon) 
=\sum_{i=1}^p\lambda_iw_{(i)}-\sum_{i=1}^p\lambda_{\pi(i)} (w_\varepsilon)_i \\
&\geq \sum_{i=1}^p\lambda_{\pi(i)}w_i-\sum_{i=1}^p\lambda_{\pi(i)} (w_\varepsilon)_i=\varepsilon\lambda_{\pi(j)}\geq0.
\end{align*}
If $\vlam \succ  0$, then the last inequality is strict.
\end{proof}

\begin{proposition}[Pulling to the mean] \label{Props3} For fixed
  sequence $\lambda_1\geq\ldots\geq\lambda_p\geq0$, let $\vw\in
  \mathbb{R}^p$ be such that $\vw \succeq 0$ and $w_j > w_l$ for some
  ${j,l\in \{1,\ldots,p\}}$. For $0 < \varepsilon \leq
  \frac{w_j-w_l}2$ define:
$$ (w_\varepsilon)_i=\left\{\begin{array}{ll} w_l+\varepsilon, & i=l\\
w_j-\varepsilon, & i=j\\
w_i,& \textrm{otherwise} \end{array} \right.. $$
Then:
\newline
(i) $J_\vlam(\vw_\varepsilon) \leq J_\vlam(\vw)$,
\newline
(ii) if $\lambda_1 > \ldots > \lambda_p$, then $J_\vlam(\vw_\varepsilon) < J_\vlam(\vw).$
\end{proposition}
\begin{proof}
  Let $\pi:\{1,\ldots,p\}\longrightarrow\{1,\ldots,p\}$ be permutation
  such as $\sum_{i=1}^p\lambda_i
  (w_\varepsilon)_{(i)}=\sum_{i=1}^p\lambda_{\pi(i)}
  (w_\varepsilon)_i$ for each $i$ in $\{1,\ldots,p\}$ and
  $\lambda_{\pi(j)}\geq\lambda_{\pi(l)}$. From the rearrangement
  inequality,
\begin{align*}
&  J_\vlam(\vw)-J_\vlam(\vw_\varepsilon)\\
&=\sum_{i=1}^p\lambda_iw_{(i)}-\sum_{i=1}^p\lambda_i(w_\varepsilon)_{(i)}\\
&=\sum_{i=1}^p\lambda_iw_{(i)}-\sum_{i=1}^p\lambda_{\pi(i)} (w_\varepsilon)_i\\
&\geq \sum_{i=1}^p\lambda_{\pi(i)}w_i-\sum_{i=1}^p\lambda_{\pi(i)} (w_\varepsilon)_i\\
&=\varepsilon\big(\lambda_{\pi(j)}-\lambda_{\pi(l)}\big)\geq0.
\end{align*}
If $\lambda_1 > \ldots > \lambda_p$, then the last inequality is strict.
\end{proof}

\begin{proposition}
\label{07141215}
Suppose that $\lambda_1\geq\ldots\geq\lambda_p\geq0$. For arbitrary
$\vx\in \mathbb{R}^p$, $\varepsilon>0$, and $l, j \in \{1,\dots,p\}$
define:
\begin{align*}
(x_\varepsilon)_i&=\left\{\begin{array}{ll} x_l+\varepsilon, & i=l\\ x_j-\varepsilon, & i=j\\x_i,& \textrm{otherwise} \end{array} \right., \\
 (\widetilde{x}_\varepsilon)_i&=\left\{\begin{array}{ll} x_j-\varepsilon, & i=j\\x_i,& \textrm{otherwise} \end{array} \right..
\end{align*}
Then:
\newline i) if $x_j>x_l\geq0$, then for $\varepsilon\in\big(0,(x_j-x_l)/2\big]$ it holds that $J_\vlam^D(\vx_{\varepsilon})\leq J_\vlam^D(\vx)$, 
\newline ii) if $x_j>0$, then for $\varepsilon\in\big(0,x_j\big]$ it holds $J_\vlam^D(\widetilde{\vx}_{\varepsilon})\leq J_\vlam^D(\vx)$.
\end{proposition}
\begin{proof}
  It is easy to observe that for any $\vx\in\mathbb{R}^p$, the dual to
  the sorted $\ell_1$-norm could be represented as
  ${J_\vlam^D(\vx)=\max\Big\{J_{\lambda^k}(x),\ k\leq p\Big\}}$ for
  $$
\lambda^k_i:=\left\{\begin{array}{cl}
      \big(\sum_{j=1}^k\lambda_j\big)^{-1},& i\leq k\\ 0,&
      \textrm{otherwise} \end{array} \right..
$$ 
The claim is therefore the straightforward consequence of Propositions
\ref{Props4} and \ref{Props3}.
\end{proof}

\begin{proposition}
\label{07141353}
If $|\vw| \preceq  |\widetilde{\vw}|$, then $J_\vlam^D(\vw)\leq J_\vlam^D(\widetilde{\vw})$.
\end{proposition}
\begin{proof}
Claim follows simply from Corollary~\ref{17021046}, analogously as in proof of previous proposition.
\end{proof}


\subsection*{Properties of the ODS with an Orthogonal Design}

Before proving Theorem~\ref{thm:ods}, we will recall results
concerning the SLOPE problem,
\begin{equation}\label{SLOPE_ortt}
 \min_{\vw} \;\; \frac{1}{2} \| \vy - \vX \vw \|^2 + J_\vlam(\vw) .
\end{equation} 
and derive some properties useful in analysis of the ODS
problem~\eqref{eq:ods}. 
Since $\frac 12\big\|\vy-\vX\vw\big\|_2^2=\frac
12\big\|\widetilde{\vy}-\vw\big\|_2^2+const$ and
$\vX^T(\vy-\vX\vw)=\widetilde{\vy}-\vw$ under orthogonal design, for
$\widetilde{\vy}:=\vX^T\vy$, without loss of generality we
will consider the case $\vX=\mathbf{I}_p$ (hereafter we also denote
$\tilde \vy$ by $\vy$ to simplify notations). Let $I_1$ be some subset
of $\{1,\ldots,p\}$ and $\overline{\vw_{I_1}}$ denotes the arithmetic
mean of subvector $\vw_{I_1}$, for any
$\vw\in\mathbb{R}^p$. \citet{BogB15} showed that the unique solution
to (\ref{SLOPE_ortt}) (uniqueness follows from strict convexity of
objective function) is output of the following
Algorithm~\ref{07171128}, which terminates in at most $n$ steps.


\begin{algorithm*}[!t]
\caption{Solution to SLOPE in the Orthogonal Case~\citep[Algorithm 3]{BogB15}}
\label{07171128}
 \textbf{Input}: Nonnegative and nonincreasing sequences $y$ and $\lambda$\;
 \While {$\vy-\vlam$ is not nonincreasing}{
   Identify strictly increasing subsequences, i.e. segments $I_i:=\{j,\ldots,l\}$ such that $$y_j-\lambda_j<\ldots<y_l-\lambda_l.$$
   For each $k\in I_i$ replace the values of $y$ and $\lambda$ by their average value over such segments $$y_k\leftarrow \overline{y_{I_i}},\qquad\lambda_k\leftarrow \overline{\lambda_{I_i}}$$
  }
\textbf{output:} $\vw_S = (\vy-\vlam)_+$\;
\end{algorithm*}

\noindent The same authors showed that (\ref{SLOPE_ortt}) can be casted as a quadratic program,
$$
\begin{array}{cl}
\min_{\vw\in\R^p} &\ \frac12\|\vy-\vw\|_2^2+\vlam^T \vw\\[.1cm]
\textrm{subject to}&\begin{array}{l}w_1\geq\ldots\geq w_p\geq0\end{array}.
\end{array}
$$

We aim to show that the ODS~(\ref{eq:ods}) can be reformulated as a
linear program, and then to show that its solution is unique and can
be given by the same Algorithm~\ref{07171128} for SLOPE.

Taking into account that $\vX=\mathbf{I}_p$, the ODS for the
orthogonal design can be rewritten as
\begin{equation}
\label{07141000}
 \min_\vw \ J_\vlam(\vw)\quad \textrm{subject to}\quad \vy-\vw\in C_{\vlam} ,
\end{equation}
where $C_\vlam$ is the unit ball in terms of the dual norm,
$J_\vlam^D$.  It is easy to see that if $\vy-\vw\in C_\vlam$, then
also $\vP_{\pi}(\vy-\vw)\in C_\vlam$ and $\vS(\vy-\vw)\in C_\vlam$ for
any permutation $\pi$ and diagonal matrix $\vS$ with $|S_{ii}|=1$,
$i=1,\ldots,p$.

\begin{proposition}
\label{07141019}
Suppose that $\vw^*$ is solution to (\ref{07141000}), $\pi$ is
arbitrary permutation of $\{1,\ldots,p\}$ and $\vS$ is diagonal matrix
such as $|S_{ii}|=1$ for all $i$. Then 
\newline i) $\vw^*_{\pi}:= \vP_{\pi} \vw^*$ is a solution to $$\min_\vw J_\vlam(\vw) \textrm{ s.t. } \vP_{\pi}\vy-\vw\in C_\vlam ,$$  
\newline ii) $\vw^*_\vS:=\vS\vw^*$ is a solution to $$\min_\vw J_\vlam(\vw)\quad \textrm{ s.t. } \vS\vy-\vw\in C_\vlam .$$
\end{proposition}
\begin{proof}
  Suppose that there exists $\vw_0\in\mathbb{R}^p$ such as
  $J_\vlam(\vw_0)<J_\vlam(\vw^*_{\pi})$ and $\vP_{\pi}\vy-\vw_0\in
  C_\vlam$. Then we have $\vy-\vP^{-1}_{\pi}\vw_0\in C_\vlam$ and
  $J_\vlam(\vP^{-1}_{\pi}\vw_0)<J_\vlam(\vw^*)$, which contradicts the
  optimality of $\vw^*$. The second part can be shown similarly.
\end{proof}

Thanks to Proposition \ref{07141019}, without loss of generality we
assume that $y_1\geq\ldots\geq y_p\geq 0$. Indeed, if for an arbitrary
$\vy\in \mathbb{R}^p$, the solution $\vw^{|\vy|_{(\cdot)}}$ to
(\ref{07141000}) for ordered magnitudes of observations is known, the
original solution could be immediately recovered by $\vw^*=
\vS\vP_{\pi}\vw^{|\vy|_{(\cdot)}}$ with $\vP_{\pi}$ and $\vS$
satisfying $\vy=\vS\vP_{\pi}|\vy|_{(\cdot)}$. This coincides with the
analogous property of the SLOPE:


\begin{proposition}
Assume that $\lambda_1>\ldots>\lambda_p>0$, $y_1\geq\ldots\geq
y_p\geq0$ and let $\vw^*$ be solution to (\ref{07141000}). Then, for any
$j\in\{1,\ldots,p\}$ we have $0\leq w^*_j\leq y_j$.
\end{proposition}
\begin{proof}
  Suppose first that for some $j$ it occurs $w_j^*<0$ and define
$$
(w_{\varepsilon})_i:=\left\{
    \begin{array}{cl} |w^*_j|-\varepsilon, & i=j\\ 
          |w^*_i|, &\textrm{otherwise} \end{array} \right..
$$ 
Fix $\varepsilon=|w_j^*|$. Then $|\vy-\vw_{\varepsilon}| \preceq |\vy-\vw^*|$,
hence Proposition \ref{07141353} yields
$J_\vlam^D(\vy-\vw_{\varepsilon})\leq J_\vlam^D(\vy-\vw^*)\leq 1$ and
$\vw_{\varepsilon}$ is feasible. Moreover, Proposition \ref{Props4}
gives $J_\vlam(\vw_{\varepsilon})<J_\vlam(|\vw^*|)=J_\vlam(\vw^*)$, which
leads to contradiction.

Suppose now that $w_j^*>y_j$. This gives that $w_j^*>0$. Define
$$
(w_{\varepsilon})_i:=\left\{\begin{array}{cl} w^*_j-\varepsilon, &
    i=j\\ w^*_i, & \textrm{otherwise} \end{array} \right.,
$$
and fix $\varepsilon:=w_j^*-y_j$. As before $\vw_{\varepsilon}$ is
feasible. Using again Proposition \ref{Props4}, we get
$J_\vlam(\vw_{\varepsilon})<J_\vlam(\vw^*)$ which contradicts the
optimality of $\vw^*$.
\end{proof}

\begin{proposition}
\label{07151336}
Assume that $\lambda_1>\ldots>\lambda_p>0$, $y_1\geq\ldots\geq
y_p\geq0$ and let $\vw^*$ be solution to (\ref{07141000}). Then it
occurs that $w_1^*\geq\ldots\geq w_p^*\geq0$ and $y_1-w_1^*\geq\ldots\geq
y_p-w_p^*\geq0$.
\end{proposition}
\begin{proof}
  From the previous propositions we have that $\vw^* \succeq \mathbf
  0$ and $\vy-\vw^* \succeq \mathbf 0$. We will show that for $1\leq
  j<l\leq p$ it holds $w_j^*\geq w_l^*$ and $y_j^*-w_j^*\geq
  y_l-w_l^*$. Suppose first that $w_j^*<w_l^*$ for some $j<l$ and
  denote $\vx:=\vy-\vw^*$. Since $y_j\geq y_l$, we have
  $x_j>x_l$. Define
\begin{align*}
(w_\varepsilon)_i &=\left\{\begin{array}{ll} w^*_j+\varepsilon, & i=j\\ w^*_l-\varepsilon, & i=l\\w^*_i,& \textrm{otherwise} \end{array} \right., \\
(x_\varepsilon)_i &=\left\{\begin{array}{ll} x_j-\varepsilon, & i=j\\x_l+\varepsilon, & i=l\\x_i,& \textrm{otherwise} \end{array} \right..
\end{align*}
Then $\vx_{\varepsilon}=\vy-\vw_{\varepsilon}$. From Propositions
\ref{Props3} and \ref{07141215}, there exist $t_1,t_2>0$ such that
$J_\vlam(\vw_{\varepsilon})<J_\vlam(\vw^*)$ for $\varepsilon\in(0,t_1)$
and $J_\vlam^D(\vx_{\varepsilon})\leq J_\vlam^D(\vx)$ for
$\varepsilon\in(0,t_2)$. Define $t:=\min\{t_1,t_2\}$ and fix
$\varepsilon \in (0,t)$. Then, we have
$J_\vlam^D(\vy-\vw_{\varepsilon})=J_\vlam^D(\vx_{\varepsilon})\leq
J_\vlam^D(\vx)\leq1$, hence $\vw_{\varepsilon}$ is feasible with smaller
value of objective.

Suppose that $y_j-w_j^*<y_l-w_l^*$ for $j<l$. This gives $x_j<x_l$ and
$w_j^*>w_l^*$. The feasible vector, $\vw_{\varepsilon}$, with smaller
value of objective, could be now constructed in an analogous manner,
again yielding the contradiction with optimality of $\vw^*$.
\end{proof}

Proposition \ref{07151336} states that including inequality
constraints $w_1\geq\ldots\geq w_p\geq0$ and $y_1-w_1\geq\ldots\geq
y_p-w_p\geq0$ to the problem~(\ref{07141000}) does not change the
set of solutions. These additional restraints simplify objective
function and as a result the task takes form of minimizing linear
function $\vlam^T\vw$. Moreover the condition $\vy-\vw\in C_\vlam$ can
now be represented by $p$ affine constraints of the form
$\sum_{i=1}^k(y_i-\lambda_i)\leq\sum_{i=1}^k w_i$. That is, the ODS
can be casted as a linear program. We will now show that after the
transformation, one can omit the conditions $y_1-w_1\geq\ldots\geq
y_p-w_p\geq0$, yielding an equivalent formulation
\begin{equation}
\label{07151335}
\begin{aligned}
\min_{\vw}& \;\;  \vlam^T\vw\\[.1cm]
\textrm{s.t.}&\left\{\begin{array}{l}\sum_{i=1}^k(y_i-\lambda_i)\leq\sum_{i=1}^k w_i,\ \ k=1,\ldots,p\\
w_1\geq\ldots\geq w_p\geq0\end{array}\right..
\end{aligned}
\end{equation}

\begin{proposition}
\label{07152152}
Let $\vw^*$ be a solution to (\ref{07151335}), for
$\lambda_1>\ldots>\lambda_p>0$ and $y_1\geq\ldots\geq y_p\geq0$. Then
$$
\begin{array}{lll}
i)\ \ \sum_{i=1}^j(y_i-\lambda_i)=\sum_{i=1}^j w^*_i\ \ \textrm{or}\ \ w^*_j=w_{j+1}^*,&&\\
ii)\ \  y_j-w^*_j\geq y_{j+1}-w^*_{j+1},&&
\end{array}
$$
for all $j\in\{1,\ldots,p\}$, with the convention that $w^*_{p+1}:=0$ and $y_{p+1}:=0$.
\end{proposition}
\begin{proof}
  Fix $j\in\{1,\ldots,p\}$ and suppose that $\vw\in\mathbb{R}^p$ is
  feasible vector of problem (\ref{07151335}) such that
  ${\sum_{i=1}^j(y_i-\lambda_i)<\sum_{i=1}^j w_i}$ and $w_j>w_{j+1}$,
  with the convention that $w_{p+1}=0$. There exists $\varepsilon>0$,
  such as
\begin{equation}
\label{07151501}
\sum_{i=1}^j(y_i-\lambda_i)<\bigg(\sum_{i=1}^j w_i\bigg)-\varepsilon\ \  \textrm{and}\ \  w_j-\varepsilon>w_{j+1}+\varepsilon.
\end{equation}
Define $\vw_{\varepsilon}\in\mathbb{R}^p$ by putting
$(w_{\varepsilon})_j:=w_j-\varepsilon$,
$(w_{\varepsilon})_{j+1}:=w_{j+1}+\varepsilon$ and
$(w_{\varepsilon})_i:=w_i$ for $i\notin\{j,j+1\}$. Thanks to
(\ref{07151501}), $\vw_{\varepsilon}$ is feasible (note that
$\sum_{i=1}^k w_i=\sum_{i=1}^k(w_{\varepsilon})_i$ for $k\neq
j$). Now, with convention $\lambda_{p+1}:=0$, it holds
$\vlam^T w-\vlam^Tw_{\varepsilon}=\varepsilon(\lambda_j-\lambda_{j+1})>0$,
which shows that $\vw$ is not optimal.

To prove ii), let $\vw$ be a feasible point such that
$y_j-w_j<y_{j+1}-w_{j+1}$ for some $j\in\{1,\ldots,p\}$. Considering
the case $j=1$, from the feasibility of $\vw$ we get
$y_1-w_1<\frac{(y_1-w_1)+(y_2-w_2)}{2}\leq\frac{\lambda_1+\lambda_2}{2}<
\lambda_1$.  For $j\in\{2,\ldots,p\}$ we have
$\sum_{i=1}^{j-1}(y_i-w_i)\leq\sum_{i=1}^{j-1}\lambda_i$,
$\sum_{i=1}^{j+1}(y_i-w_i)\leq\sum_{i=1}^{j+1}\lambda_i$ (with
$\lambda_{p+1}:=0$). Adding both sides of these inequalities and
dividing by $2$ yields
\begin{align*}
&  \sum_{i=1}^{j-1}(y_i-w_i)+\frac{(y_j-w_j)+(y_{j+1}-w_{j+1})}{2} \\
& \leq \sum_{i=1}^{j-1}\lambda_i+\frac{\lambda_j+\lambda_{j+1}}{2} 
<\sum_{i=1}^j\lambda_i.
\end{align*}
Due to the assumption $y_j-w_j<y_{j+1}-w_{j+1}$, it follows that 
\begin{align*}
&\sum_{i=1}^j(y_i-w_i) \\
& <\sum_{i=1}^{j-1}(y_i-w_i)+\frac{(y_j-w_j)+(y_{j+1}-w_{j+1})}{2} 
<\sum_{i=1}^j\lambda_i.
\end{align*}
To sum up, we always have $\sum_{i=1}^j(y_i-\lambda_i)<\sum_{i=1}^j
w_i$. Moreover, $y_j\geq y_{j+1}$ and $y_j-w_j<y_{j+1}-w_{j+1}$ give
that $w_j>w_{j+1}$. Therefore, from (i), the vector $\vw$ can not be
optimal.
\end{proof}

We now show that the LP~(\ref{07151335}) has a unique solution.  For
$k\in\mathbb{N}$, define $k \times k$ upper and lower triangular
matrices $\vS_k$ and $\vV_k$ as
\begin{equation}
\label{07221432}
\begin{aligned}
\big(S_k\big)_{j,l}&=
\left\{ \begin{array}{rl}
1,&j\leq l\\
0, & otherwise\\
\end{array} \right.,
\\
\big(V_k\big)_{j,l}&=
\left\{ \begin{array}{rl}
1,&l=j\\
-1, & j=l+1\\
0, & otherwise
\end{array} \right..
\end{aligned}
\end{equation}
It could be easily verified, that $\vS_k^{-1}=\vV_k^T$. We are now ready to prove the following lemma.
\begin{lemma}
  The LP~(\ref{07151335}) has a unique solution for $\lambda_1 >
  \cdots > \lambda_p > 0$.
\end{lemma}
\begin{proof}
  Denote the columns of matrices $\vS_p$ and $\vV_p$ by
  $\vs_1,\ldots,\vs_p$ and $\vv_1,\ldots,\vv_p$, respectively. From
  the condition $\vS_p\vV_p^T=\mathbf{I}_p$, we get that $\vs_i$ is
  orthogonal to $\vv_j$ whenever $i\neq j$. Hence, since $\vS_p$ and
  $\vV_p$ are nonsingular, the matrix $\big[(\vS_p)_{I_1}\big|
  (\vV_p)_{I_2}\big]$ is nonsingular as well, for any partition
  $\{I_1,I_2\}$ of the set $\{1,\ldots,p\}$. This means that the set
\begin{align*}
\text{SOL} &:=\Big\{\vw\in\mathbb{R}^p:\ \big[(\vS_p)_{I_1}\big| (\vV_p)_{I_2}\big]^T \vw=c^{I_1,I_2}, \\
 &\quad  I_1\cup I_2=\{1,\ldots,p\},\ I_1\cap I_2=\emptyset\Big\}
\end{align*}
is finite, where $c^{I_1,I_2}_j:=\sum_{i=1}^j(y_i-\lambda_i)$, for $j\in I_1$, and $c^{I_1,I_2}_j:=0$, for $j\in I_2$. Let $\vw^*$ be any solution to the
ODS~(\ref{07151335}). From Proposition \ref{07152152} i), for all
$j\in\{1,\ldots,p\}$ we have $\vs_j^T\vw^*=0$ or $\vv_j^T\vw^*=0$,
which gives that $\vw^*\in \text{SOL}$. Since a feasible LP can have
either one of infinitely many solutions, this immediately gives the
claim.
\end{proof}


\begin{lemma}
\label{07202218}
Consider perturbed version of the LP~(\ref{07151335}) with the same
feasible set but with a new objective function
$f_{\mu}(\vw):=\vlam^T\vw+\frac12\mu\|\vw\|^2_2$ with $\mu>0$. Let
$\vw^*$ be solution to perturbed problem (which is unique thanks to
strong convexity). Then for any $y_1\geq\ldots\geq y_p\geq 0$ and
$\lambda_1\geq\ldots\geq \lambda_p\geq 0$ (i.e. coefficients of
$\vlam$ do not have to be strictly decreasing and positive), it occurs
${w^*_j=w_{j+1}^*}\ \ \textrm{or}\ \
{\sum_{i=1}^j(y_i-\lambda_i)=\sum_{i=1}^j w^*_i}$ for all
$j\in\{1,\ldots,p\}$ (with $w^*_{p+1}:=0$ and $y_{p+1}:=0$).
\end{lemma}
\begin{proof}
  Take any feasible $\vw$ and assume that for some
  $j\in\{1,\ldots,p\}$ we have $w_j>w_{j+1}$ and
  ${\sum_{i=1}^j(y_i-\lambda_i)<\sum_{i=1}^j w_i}$. Let
  $\vw_{\varepsilon}$ be feasible vector constructed as in proof of
  Proposition \ref{07152152} i). Then
\begin{align*}
&f_{\mu}(\vw)-f_{\mu}(\vw_{\varepsilon})\\
&= \varepsilon(\lambda_j-\lambda_{j+1})+\frac12\mu\big(w_j^2+w_{j+1}^2\big)\\
& \quad -\frac12\mu\big((w_j-\varepsilon)^2+(w_{j+1}+\varepsilon)^2\big)\\
&=\varepsilon(\lambda_j-\lambda_{j+1})+\mu\varepsilon\big((w_j-w_{j+1})-\varepsilon\big)>0,
\end{align*}
for sufficiently small $\varepsilon>0$. Hence, $\vw$ can not be optimal.
\end{proof}

\begin{lemma}
\label{07171426}
Consider perturbed version of the LP (\ref{07151335}) as in the
previous lemma, with objective function $f_{\mu}(\vw)$
and a solution $\vw^*$. Moreover, assume that for some
$j,l\in\{1,\ldots,p\}$, $j<l$ we have $y_j-\lambda_j\leq
y_{j+1}-\lambda_{j+1}\leq \ldots \leq y_l-\lambda_l$. Let $I_1$ denote
the set $\{j,\ldots,l\}$ and $\overline{\vw_{I_1}}$ denote the
arithmetic mean of subvector $\vw_{I_1}$, for any
$\vw\in\mathbb{R}^p$. Then for any $y_1\geq\ldots\geq y_p\geq 0$ and
$\lambda_1\geq\ldots\geq \lambda_p\geq 0$: \vspace{5pt}
\newline
\indent i) solution is constant on the segment $I_1$,
i.e. $w_j^*=w_{j+1}^*=\ldots=w_l^*$, 
\newline 
\indent ii) $\vw^*$ is solution to perturbed problem with $\vy$ and $\vlam$ replaced
respectively by $\widetilde{\vy}$ and $\widetilde{\vlam}$,
\vspace{5pt}\newline where
\begin{equation}
\label{07211047}
\begin{array}{lllr}
\widetilde{\lambda}_i:=\left\{\begin{array}{ll} \overline{\lambda_{I_1}}, & i\in I_1\\ \lambda_i,&\textrm{otherwise} \end{array} \right., 
& \widetilde{y}_i:=\left\{\begin{array}{ll} \overline{y_{I_1}}, & i\in I_1\\ y_i,&\textrm{otherwise} \end{array} \right..
\end{array}
\end{equation}
\end{lemma}

\begin{proof}
  To prove i), suppose that $w_k^*>w_{k+1}^*$ for
  $k\in\{j,\ldots,l-1\}$. Using the convention that
  $y_0:=w_0^*:=\lambda_0:=0$, from feasibility of $\vw^*$ we have
$$
\sum_{i=0}^{k-1}(y_i-\lambda_i)\leq\sum_{i=0}^{k-1}w^*_i\ \ \textrm{and}\ \ \sum_{i=0}^{k+1}(y_i-\lambda_i)\leq\sum_{i=0}^{k+1}w^*_i.
$$
Adding both sides of these inequalities and dividing by $2$ yields
\begin{align*}
&\sum_{i=0}^{k-1}(y_i-\lambda_i)+\frac{(y_k-\lambda_k)+(y_{k+1}-\lambda_{k+1})}{2}\\
&\leq\sum_{i=0}^{k-1}w^*_i +\frac{w^*_k+w^*_{k+1}}{2}<\sum_{i=1}^k w^*_i.
\end{align*}
Therefore
\begin{align*}
\sum_{i=1}^k(y_i-\lambda_i) &\leq\sum_{i=0}^{k-1}(y_i-\lambda_i)+\frac{(y_k-\lambda_k)+(y_{k+1}-\lambda_{k+1})}{2}\\
&<\sum_{i=1}^k w^*_i,
\end{align*}
which yields contradiction with Lemma \ref{07202218}.

To show that the aforementioned modification of $\vlam$ and $\vy$ does
not affect the solution, we will first show that feasible sets of both
problems are identical. Let $D$ and $\widetilde{D}$ denote the
feasible sets for, respectively, initial parameters ($\vy,\vlam)$ and
($\widetilde{\vy},\widetilde{\vlam})$ given by (\ref{07211047}). We
start with proving that $\widetilde{D}\subset D$. Let $\vw$ be any
vector from $\widetilde{D}$. Since
$\sum_{i=1}^k(y_i-\lambda_i)=\sum_{i=1}^k(\widetilde{y}_i-\widetilde{\lambda}_i)$,
for $k<j$ and $k\geq l$, the task reduces to showing that
$\sum_{i=1}^k w_i\geq\sum_{i=1}^k(y_i-\lambda_i)$ for any
$k\in\{j,\ldots,l-1\}$. Since $\{y_i-\lambda_i\}_{i=j}^l$ increases,
we simply have
$\overline{y_{I_1}}-\overline{\lambda_{I_1}}\geq\frac1{k-j+1}\sum_{i=j}^k(y_i-\lambda_i)$. Using
the convention $y_0:=\lambda_0:=0$, we get
\begin{align*}
&  \sum_{i=1}^k w_i\geq \sum_{i=1}^k(\widetilde{y}_i-\widetilde{\lambda}_i) \\
&= \sum_{i=0}^{j-1}(y_i-\lambda_i) + (k-j+1)\cdot\Big(\overline{y_{I_1}}-\overline{\lambda_{I_1}}\Big)\geq\sum_{i=1}^k(y_i-\lambda_i).
\end{align*}
Now, take any $\vw\in D$. After defining $w_0:=0$ we have
${(l-k)\cdot\sum_{i=0}^{j-1} w_i\geq(l-k)\cdot\sum_{i=0}^{j-1}(y_i-\lambda_i)}$
and
${(k-j+1)\cdot\sum_{i=0}^l w_i\geq(k-j+1)\cdot\sum_{i=0}^l(y_i-\lambda_i)}$.
Adding both sides of these inequalities and dividing by $(l-j+1)$ yields
$$
\sum_{i=0}^{j-1}w_i+(k-j+1)\cdot\overline{\vw_{I_1}}
\geq \sum_{i=0}^{j-1}(y_i-\lambda_i)+(k-j+1)\cdot\big(\overline{\vy_{I_1}}-\overline{\vlam_{I_1}}\big).
$$
From the monotonicity of $\vw$, it occurs $\sum_{i=j}^kw_i\geq(k-j+1)\cdot\overline{\vw_{I_1}}$. Consequently,
\begin{align*}
&\sum_{i=1}^kw_i\geq \sum_{i=0}^{j-1}w_i +(k-j+1)\cdot\overline{\vw_{I_1}}\\
&\geq \sum_{i=0}^{j-1}(y_i-\lambda_i)+(k-j+1)\cdot\big(\overline{\vy_{I_1}}-\overline{\vlam}_{I_1}\big)\\
&=\sum_{i=1}^k(\widetilde{y}_i-\widetilde{\lambda}_i)
\end{align*}
and $D\subset\widetilde{D}$ as a result.

Suppose now that $\vw^*$ is solution for initial parameters
($\vy,\vlam$), $\widetilde{\vb}^*$ is solution for
($\widetilde{\vy},\widetilde{\vlam})$ and that
\begin{equation}
\label{07211530}
\frac12\mu\|\widetilde{\vw}^*\|^2_2+\widetilde{\vlam}^T\widetilde{\vw}^*<\frac12\mu\|\vw^*\|^2_2+\widetilde{\vlam}^T\vw^*
\end{equation}
From i) we have $w_j^*=\ldots=w_l^*$ and
$\widetilde{w}_j^*=\ldots=\widetilde{w}_l^*$, which yields
$\widetilde{\vlam}^T\widetilde{\vw}^*=\vlam^T\widetilde{\vw}^*$ and
$\widetilde{\vlam}^Tb^*=\vlam^T\vw^*$. Therefore from
(\ref{07211530}) we have $f_{\mu}(\widetilde{\vw}^*)< f_{\mu}(\vw^*)$,
which contradicts the optimality of $\vw^*$.
\end{proof}

\subsubsection*{Proof of Theorem~\ref{thm:ods}}

Without loss of generality we can assume that we are starting with
ordered and nonzero observations. Basing on Propositions
\ref{07151336} and \ref{07152152}, each solution to (\ref{eq:ods}) is
also a solution to (\ref{07151335}). Since such solution is unique,
this immediately gives the uniqueness of (\ref{eq:ods}). Consider
perturbed version of (\ref{07151335}), with objective $f_{\mu}$ for
sufficiently small $\mu$, such as solutions to (\ref{07151335}) and
its perturbed version coincide (the existence of such $\mu$ is
guaranteed by \citep[Theorem 1]{BecC11}.
Modifying $\vy$ and $\vlam$ as in
  Algorithm~\ref{07171128}, after finite number of iterations we
  finish with converted $\vy$ and $\vlam$ such as
\begin{equation}
\label{07171430}
y_1-\lambda_1\geq\ldots\geq y_p-\lambda_p.
\end{equation}%
From Lemma \ref{07171426}, we know that such modifications do not have
an impact on the solution. Therefore, it is enough to show that, when
assumption (\ref{07171430}) is in use, the solution to SLOPE,
i.e. $\vw_S = (\vy-\vlam)_+$, is also the unique solution, to
(\ref{07151335}).

With $\vS_p$ and $\vV_p$ defined in (\ref{07221432}), the perturbed problem
has following convex optimization form with affine inequality
constraints
\begin{equation}
\label{07201716}
\begin{aligned}
\min_\vw & \;\;  \frac12\mu\|\vw\|_2^2+\vlam^T\vw\\
\textrm{s.t.} &\;\; \left\{\begin{array}{l}\vS_p^T(\vy-\vlam-\vw) \preceq 0,\\ -\vV_p^T\vw \preceq  0\end{array}\right..
\end{aligned}
\end{equation}
If $\vy-\vlam \prec 0$, put $I_1:=\emptyset$,
$I_2:=\{1,\ldots,p\}$. Otherwise, let $s$ be the maximal index such
that $y_s-\lambda_s\geq0$ and define $I_1:=\{1,\ldots,s\}$,
$I_2:=\{1,\ldots,p\}\backslash I_1$. The KKT conditions for
(\ref{07201716}) are given by
$$
\begin{cases}
  \mu \vw+\vlam=\vS_p\vnu+\vV_p\vtau, \;\; \text{(Stationary)}\\
  \nu_i\big(\vS_p^T(\vy-\vlam-\vw)\big)_i=0,\ \tau_i\big(\vV_p^T\vw\big)_i=0 \\
   \qquad \textrm{for each}\ i\in\{1,\ldots,p\}, \;\; \text{(Complementary slackness)}\\ 
  \vS_p^T(\vy-\vlam-\vw)\preceq 0,\ -\vV_p^T\vw\preceq 0, \;\; \text{(Primal feasibility)}\\
  \vnu\succeq 0,\ \vtau\succeq 0. \;\; \text{(Dual feasibility)}
\end{cases}
$$
We now show that $(\vw^*,\vnu^*,\vtau^*)$ satisfy the KKT conditions,
where $\vw^*=(\vy-\vlam)_+$ and $\vnu^*$, $\vtau^*$ are given
by 
\begin{align*}
\vnu^*_{I_1}&:=\vV_s^T(\mu \vw^*+\vlam)_{I_1}, \; 
\vnu^*_{I_2}:=\mathbf{0}, \; 
\vtau^*_{I_1}:=\mathbf{0}, \;\\
 \vtau^*_{I_2}&:=\vS_{p-s}^T(\mu \vw^*+\vlam)_{I_2}.
\end{align*}
It is easy to see that $\vw^*$ is primal feasible. Since coefficients
of $\mu \vw^*+\vlam$ create a nonnegative and nonincreasing
sequence, we have $\vnu^*\succeq 0,\ \vtau^*\succeq 0$. Moreover,
$(\vS_p)_{I_1}^T(\vy-\vlam-\vw^*)=0$ and $(\vV_p)^T_{I_2}b^*=0$, which shows
that complementary slackness conditions are satisfied. Furthermore, we
have
\begin{align*}
&\vS_p \vnu^*+ \vV_p\vtau^*=\vS_{I_1}\vnu^*_{I_1}+\vV_{I_2}\vtau^*_{I_2} \\
&=\left[\begin{BMAT}(b,2pt,11pt){c0c}{c0c}\vS_s&0\\ 0& \vV_{p-s}\end{BMAT}\right] 
\left[\begin{BMAT}(b,2pt,10pt){c}{c0c}\vnu^*_{I_1}\\ \vtau^*_{I_2}\end{BMAT}\right] = \mu \vw^*+\vlam,
\end{align*}
which shows that stationary condition is met and finishes the proof. \qed

\subsection*{HPE Algorithm}

For a reference, we provide the HPE algorithm. For the GDS problem, we use $f=0$, $L_f=0$, $g_1 = \cF$, and $g_2 = \cG$.

\begin{algorithm*}[!t]
\SetKwComment{tpa}{(}{)}
\SetKwInOut{Input}{input}
\SetKwInOut{Output}{output}
\caption{HPE Algorithm}\label{alg:hpe}
\Input{$(x_0, y_0) \in \mathcal X, \mathcal Y$, $\eta>0$, and $\sigma \in (0,1)$.}

\For {$k=1,2,\dots$} {
/* Solve the HPE error condition */
$$
  (\tilde u, \tilde v, \tilde r^u, \tilde r^v, \tilde\epsilon) = \text{HPE-Error-Cond}(f=\eta f, A=\eta A, g_1 = \eta g_1, g_2 = \eta g_2, (u_0,v_0) = (x_{k-1}, y_{k-1}), L_f = \eta L_f)
$$

/* Update */
\begin{align*}
& (\tilde x_k, \tilde y_k ) := (\tilde u, \tilde v), \;\; 
 \tilde r_k = (\tilde r^x_k, \tilde r^y_k) := \frac{1}{\eta} (\tilde r^u, \tilde r^v), \;\;
 \epsilon_k = \frac{1}{\eta} \tilde\epsilon \\
& x_k = x_{k-1} - \eta \tilde r^x_k, \;\; y_k = y_{k-1} - \eta \tilde r^y_k
\end{align*}

/* Check Convergence */
 }
\end{algorithm*}

\begin{algorithm*}[!t]
\SetKwComment{tpa}{(}{)}
\SetKwInOut{Input}{input}
\SetKwInOut{Output}{output}
\caption{HPE-Error-Cond$(f, A, g_1, g_2, (u_0,v_0), L_f)$ Subroutine}\label{alg:apel}
\Input{$f$, $A$, $g_1$, $g_2$, $(u_0, v_0)$, and $L_f$.}

Set $L = L_f + \|A\|^2$, $t_0 = 0$, $\tilde u_0 = w_0 = P_\Omega(u_0)$, and $\tilde v_0=0$.

\For {$k=1,2,\dots$} {
/* Update */
$$ 
  t_k = t_{k-1} + \frac{1+\sqrt{1+4L t_{k-1}}}{2L}
$$
$$
  u_k = \frac{t_{k-1}}{t_k} \tilde u_{k-1} + \frac{t_k - t_{k-1}}{t_k} w_{k-1}
$$
$$
  v(u_k) = \argmax_v \;\; \langle Au, v \rangle - g_2(v) - \frac{1}{2} \|v-v_0\|^2 = \prox_{g_2} (v_0 + Au) 
$$
$$
 \tilde v_k = \frac{t_{k-1}}{t_k} \tilde v_{k-1} + \frac{t_k - t_{k-1}}{t_k} v(u_k)
$$
$$
 w_k = \argmin_u \;\; \langle A^*\tilde v_k, u \rangle + g_1(u) + \frac{c_k}{2} \|u-u_0\|^2 = \prox_{c_k^{-1}g_1} (u_0 - c_k^{-1} A^*\tilde v_k ), \;\; c_k = 1 + \frac{1}{t_k}
$$
$$
 \tilde u_k = \frac{t_{k-1}}{t_k} \tilde u_{k-1} + \frac{t_k - t_{k-1}}{t_k} w_k
$$

/* Set */
$$
\tilde \epsilon_k = \frac{1}{2t_k} \|\tilde u_k - u_0\|^2, \;\; \tilde r^u_k = c_k(u_0-w_k), \;\; \tilde r^v_k = v_0 - v(\tilde u_k) = v_0 - \prox_{g_2} (v_0 + A \tilde u_k)
$$

/* Check Convergence */
\If{$\|\tilde r^u_k + \tilde u_k - u_0 \|^2 + \| \tilde r^v_k + \tilde v_k - v_0\|^2 + 2\tilde\epsilon_k \le \sigma^2 (\|\tilde u_k - u_0\|^2 + \|\tilde v_k -v_0\|^2)$}{ stop.}
}
\end{algorithm*}

\fi 

\end{document}